\documentclass[twocolumn,10pt]{IEEEtran}

\usepackage{setspace}
\usepackage{epsfig}  % for figures
\usepackage{graphicx}  % another package that works for figures
\usepackage{amsmath}  % for math spacing
\usepackage{amssymb}  % for math spacing
\usepackage{amsthm}
\usepackage{url}  % Hyphenation of URLs.
\usepackage{lscape}  % Useful for wide tables or figures.
\usepackage[justification=raggedright]{caption}	% makes captions ragged right - thanks to Bryce Lobdell
\usepackage{bm} %Bold numbers
\usepackage{bbm}
\usepackage{algorithm,algorithmic} % Algorithm enivronment
\usepackage{epstopdf} %Allow eps
\usepackage[T1]{fontenc} %Ensure that fonts are embedded
\usepackage{float}
\usepackage{cite}
\usepackage{mathrsfs}
\usepackage{enumerate}

\usepackage{color}

\usepackage[font=footnotesize]{subfig}
\usepackage{amssymb}
\usepackage{etoolbox}

	\usepackage{tikz}
	
	\usetikzlibrary{calc,trees,positioning,arrows,chains,shapes.geometric,%
		decorations.pathreplacing,decorations.pathmorphing,shapes,%
		matrix,shapes.symbols,positioning,fit}
	
	\usepackage{tabularx}
	
	\usepackage{mathptmx}
	\usepackage{enumitem, hyperref}
	
	\makeatletter
	\def\namedlabel#1#2{\begingroup
		#2%
		\def\@currentlabel{#2}%
		\phantomsection\label{#1}\endgroup
	}
	
	\tikzstyle{decision} = [diamond, draw, fill=blue!20,
	text width=4.5em, text badly centered, node distance=3cm, inner sep=0pt]
	\tikzstyle{block} = [rectangle, draw, fill=none,
	text width=6em, text centered, rounded corners, minimum height=4em]
	\tikzstyle{line} = [draw, -latex']
	\tikzstyle{cloud} = [draw, ellipse,fill=red!20, node distance=3cm,
	minimum height=2em]

	\newcounter{dir_est_ineq_cnt}

	\newcolumntype{L}[1]{>{\raggedright\arraybackslash}p{#1}}
	\newcolumntype{C}[1]{>{\centering\arraybackslash}p{#1}}
	\newcolumntype{R}[1]{>{\raggedleft\arraybackslash}p{#1}}

	\newtheorem{thm}{Theorem}

	\def\bx{\bm{x}}
	
	\def\bz{\bm{z}}
	\def\bw{\bm{w}}

	\def\xSp{\mathcal{X}}

	\newcommand{\inprod}[2]{\left\langle #1 , #2 \right\rangle }

	\newcommand{\sgrad}[3]{\nabla_{\bw} \ell_{#3}\left(#1,#2\right)}

	\newtoggle{useMeanGap}
	\togglefalse{useMeanGap}
	\iftoggle{useMeanGap}{

	}{

}

\newtoggle{useTwoColumn}
\toggletrue{useTwoColumn}

\iftoggle{useTwoColumn}{
	
}{

}

\title{Adaptive Sequential Machine Learning}

\author{
	\IEEEauthorblockN{Craig Wilson, Yuheng Bu and Venugopal Veeravalli}\thanks{Research reported in the paper was supported by the NSF under award CCF 11-11342, and by the Army Research Laboratory under Cooperative Agreement W911NF-17-2-0196, through the University of Illinois at Urbana-Champaign. Part of this work was presented in  ICASSP 2016 \cite{Wilson2015a} and Asilomar Conference 2016 \cite{adaptive2016}. Craig Wilson is now at Google.}\\
	\IEEEauthorblockA{University of Illinois at Urbana-Champaign
		\\\{wilson60, bu3, vvv\}@illinois.edu}\\
}

\begin{document}

\maketitle
\begin{abstract}
A framework previously introduced in \cite{Wilson2016a} for solving a sequence of stochastic optimization problems with bounded changes in the minimizers is extended and applied to machine learning problems such as regression and classification. The stochastic optimization problems arising in these machine learning problems is solved using algorithms such as stochastic gradient descent (SGD). A method based on estimates of the change in the minimizers and properties of the optimization algorithm is introduced for adaptively selecting the number of samples at each time step to ensure that the excess risk, i.e., the expected gap between the loss achieved by the approximate minimizer produced by the optimization algorithm and the exact minimizer, does not exceed a target level.  A bound is developed to show that the estimate of the change in the minimizers is non-trivial provided that the excess risk is small enough. Extensions relevant to the machine learning setting are considered, including a cost-based approach to select the number of samples with a cost budget over a fixed horizon, and an approach to applying cross-validation for model selection. Finally, experiments with synthetic and real data are used to validate the algorithms.
\end{abstract}

\section{Introduction}
Consider solving a sequence of machine learning problems by minimizing the risk, i.e., expected value of a fixed loss function $\ell(\bw,\bz)$ at each time $n$:
\begin{equation}
\label{probState:seqProbs}
	\min_{\bw \in \xSp} \left\{ f_{n}(\bw) \triangleq \mathbb{E}_{\bz_{n} \sim p_{n}} \left[  \ell(\bw,\bz_{n}) \right] \right\} \;\;\; \forall n \geq 1
\end{equation}
where $p_n$ denotes the underlying (unknown) probabilistic model for the data $\bz_n$ at time $n$. For regression, $\bz_{n} = \{\bx_n, y_n\}$ corresponds to the \{predictors, response\} pair at time $n$ and $\bw$ parameterizes the regression model. For classification, $\bz_{n} = \{\bx_n, y_n\}$ corresponds to the \{features, label\} pair at time $n$, and $\bw$ parameterizes the classifier. Although, motivated by regression, and classification, our framework works for any loss function $\ell(\bw,\bz)$ that satisfies certain properties discussed in Section~\ref{withRhoKnown}.

We assume that the change in the problems is bounded by imposing a condition on the minimizers $\bw_{n}^{*}$ of the function $f_{n}(\bw)$. We assume that the problems change at a bounded but unknown rate:
\begin{equation}
\label{probState:slowChangeDef}
	\| \bw_{n}^{*} - \bw_{n-1}^{*} \| \leq \rho, \qquad  \forall n \geq 2.
\end{equation}
The value of $\rho$ is unknown to us.

Under this model, we find approximate minimizers $\bw_{n}$ of each function $f_{n}(\bw)$ by drawing $K_{n}$ samples $\{\bz_{n}(k)\}_{k=1}^{K_{n}} \overset{\text{iid}}{\sim} p_{n}$ at time $n$. We do not make any assumptions about the particular optimization algorithm that may be used to find the approximate minimizers. As an example, we could use these samples in an optimization algorithm such as SGD. We evaluate the quality of our approximate minimizers $\bw_{n}$ through an excess risk criterion $\epsilon$, i.e.,
\begin{equation}
\label{probState:meanGapDef}
		\mathbb{E}\left[ f_{n}(\bw_{n}) \right] - f_{n}(\bw_{n}^{*}) \leq \epsilon
\end{equation}
which is a standard criterion for optimization and learning problems \cite{Mohri2012}. Our goal is to determine adaptively the number of samples $K_{n}$ required to achieve a desired excess risk $\epsilon$ for large enough $n$ with $\rho$ unknown. As $\rho$ is unknown, we will first construct an estimate of $\rho$. Given an estimate of $\rho$, we determine selection rules for the number of samples $K_{n}$ to achieve a target excess risk $\epsilon$.

%the conference versions of the paper do not contain proofs of the key results, and there are substantially more detailed numerical results and simulations in the journal submission relative to the conference papers.

This paper is  a continuation of the work initiated in \cite{Wilson2016a}. We specialize the results in \cite{Wilson2016a}, which were given for general functions $f_{n}(\bw)$, to the specific form in \eqref{probState:seqProbs}, and provide new results that are specifically relevant to machine learning problems. We develop a bound to show that our estimate $\rho$  is non-trivial provided that the excess risk is small enough. We also consider extensions relevant to the machine learning setting, including a cost-based approach to select the number of samples with a cost budget over a fixed horizon, and an approach to applying cross-validation for model selection. Some of the results in this paper have reported in conference publications \cite{Wilson2015a} and \cite{adaptive2016}, which do not contain proofs of the key results due to space limitations. Moreover, we provide substantially more detailed numerical results and simulations in this paper than those given in \cite{Wilson2015a} and \cite{adaptive2016}.

\subsection{Related Work}

Our problem has connections with \textit{multi-task learning} (MTL) and \textit{transfer learning}. In multi-task learning, one tries to learn several tasks simultaneously as in \cite{Agarwal2011},\cite{Evgeniou2004}, and \cite{Zhang2012} by exploiting the relationships between the tasks. In transfer learning, knowledge from one source task is transferred to another target task either with or without additional training data for the target task \cite{Pan2010}. For multi-task and transfer learning, there are theoretical guarantees on regret for some algorithms \cite{Agarwal2008}. Multi-task learning could be applied to our problem by running a MTL algorithm each time a new task arrives, while remembering all prior tasks. However, this approach incurs a memory and computational burden. Transfer learning lacks the sequential nature of our problem.

We can also consider the \textit{concept drift} problem in which we observe a stream of incoming data that potentially changes over time, and the goal is to predict some property of each piece of data as it arrives. After prediction, we incur a loss that is revealed to us. For example, we could observe a feature $\bx_{n}$ and predict the label $y_{n}$ as in \cite{Towfic2013}. Some approaches for concept drift use iterative algorithms such as SGD, but without specific models on how the data changes. As a result, only simulation results showing good performance are available.
%
%There are also some bandit approaches in which one of a finite number of predictors must be applied to the data as in \cite{Tekin2014}. For this approach, there are regret guarantees using techniques for analyzing bandit problems.

Another related problem is online optimization, where generally no knowledge is available about the incoming functions other than that all the functions come from a specified class of functions, i.e., linear or convex functions with uniformly bounded gradients. Online optimization models do not include the notion of a desired excess risk bound. Rather, only bounds on the regret over some time horizon have been investigated~\cite{Cesa2006,Duchi2011,Duchi2009,Hazan2007,Bartlett2008,Shwartz2009,Shwartz2006,Shwartz2007,Xiao2010,Zinkevich2003}, which is different from the per time-step excess risk guarantee provided in our work.

There has been some work on controlling the variation of the sequence of functions $f_{n}(\bw)$ in \eqref{probState:seqProbs} in~\cite{RakhlinSridharan2012} and~\cite{Yang2012}. The work in~\cite{Yang2012} is most relevant where regret is minimized subject to a bound, say $G_{b}$, on the total variation of the gradients over a time interval $T$ of interest, i.e.,
\begin{equation}
\label{intro:yang_var}
\sum_{n=2}^{T} \max_{\bw \in \xSp} \| \nabla f_{n+1}(\bw) -  \nabla f_{n}(\bw) \|_{2}^{2} \leq G_{b}.
\end{equation}
If the functions $\{f_{n}(x)\}$ are strongly convex with the same parameter $m$, then by the optimality conditions (see Theorem 2F.10 in~\cite{Dontchev2009}) \eqref{intro:yang_var} implies that
\[
\sum_{n=2}^{T} \|\bw_{n+1}^{*} - \bw_{n}^{*}\|_{2}^{2} \leq \frac{G_{b}}{m^2}.
\]
Thus, the work in~\cite{Yang2012} can be seen as studying the regret with a constraint on the total variation in the minimizers over $T$ time instants. In contrast, we control the variation of the minimizers at each time instant with \eqref{probState:slowChangeDef} and then seek to maintain an excess risk criterion such as \eqref{probState:meanGapDef} at each time instant.

Another relevant model is \textit{sequential supervised learning} (see \cite{Dietterich2002}) in which we observe a stream of data consisting of feature/label pairs $(\bx_{n},y_{n})$ at time $n$, with $\bx_{n}$ being the feature vector and $y_{n}$ being the label. At time $n$, we want to predict $y_{n}$ given $\bw_{n}$. One approach to this problem, studied in \cite{Fawcett1997} and \cite{Qian1988}, is to look at $L$ consecutive pairs $\{(\bx_{n-i},y_{n-i})\}_{i=1}^{L}$ and develop a predictor at time $n$ by applying a supervised learning algorithm to this training data. Another approach is to assume that there is an underlying hidden Markov model (HMM) governing the data \cite{BengioFrasconi1996}. The label $y_{n}$ represents the hidden state and the pair $(\bx_{n},\overline{y}_{n})$ represents the observation with $\overline{y}_{n}$ being a noisy version of $y_{n}$. HMM inference techniques are used to estimate $y_{n}$.

The adaptation that we discuss in the paper is similar in spirit  to that in prior work in adaptive signal processing (see, e.g., \cite{Solo1995,Sayed2008,TYS2010}), but the techniques that we use are substantively different.

To summarize, none of the prior work discussed in this section involves choosing the number of samples $K_{n}$ at each time $n$ to control the excess risk. Most approaches instead focus on bounding the regret or provide no guarantees.

\vspace{-0.08cm}

\subsection{Paper Outline}
The rest of this paper is outlined as follows. In Section~\ref{asoPriorWork}, we specialize the work in \cite{Wilson2016a} to the machine learning problem stated in \eqref{probState:seqProbs}. In Section~\ref{withRhoKnown}, we consider the problem of minimizing the sequence of functions in \eqref{probState:seqProbs} with $\rho$ from \eqref{probState:slowChangeDef} known. In Section~\ref{estRho}, we introduce a method to estimate $\rho$. In Section~\ref{withRhoUnknown}, we consider solving the sequence of learning problems in \eqref{probState:seqProbs} with $\rho$ unknown. In Section~\ref{rhoBoundOvershoot}, we develop an upper bound on the size of the overshoot of our estimate of $\rho$ above the true value of $\rho$. In Section~\ref{ext}, we consider a cost based approach to select the number of samples based on the analysis in Section~\ref{asoPriorWork}, and a cross-validation approach. Finally, in Section~\ref{exper}, we apply our framework to a variety of machine learning problems on both synthetic and real data.

\section{Adaptive Sequential Optimization}
\label{asoPriorWork}
We summarize our previous work in \cite{Wilson2016a},  and apply it to the machine learning problem stated in \eqref{probState:seqProbs}.

\subsection{Assumptions}

We make several assumptions to proceed. First, let $\xSp$ be closed and convex with $\text{diam}(\xSp) < + \infty$. Define the $\sigma$-algebra
\begin{equation}
	\label{withRhoUnknown:FSigAlg}
	%\mathcal{F}_{i} \triangleq \sigma\left( \bigcup_{j=1}^{i} \left\{ \bz_{j}(k) \right\}_{k=1}^{K_{j}}  \right)
	\mathcal{F}_{i} \triangleq \sigma\left(  \left\{ \bz_{j}(k): \;  j = 1, \ldots, i; \; k=1, \ldots, K_{j} \right\} \right)
	\end{equation}
	which is the smallest $\sigma$-algebra such that the random variables in the set $\left\{ \bz_{j}(k): \;  j = 1, \ldots, i; \;k=1, \ldots, K_{j} \right\}$ are measurable. By convention $\mathcal{F}_{0}$ is the trivial $\sigma$-algebra.

We suppose that the following conditions hold:
\begin{description}
	\item\namedlabel{probState:assump1}{A.1}
	For each $n$, $f_{n}(\bw)$ is twice continuously differentiable with respect to $\bw$.
	\item\namedlabel{probState:assump2}{A.2}
	For each $n$, $f_{n}(\bw)$ is strongly convex with a parameter $m>0$, i.e.,
	\begin{equation}
	\label{prob_form:strong_convex_cond}
	f_{n}(\tilde{\bw}) \geq f_{n}(\bw) + \inprod{\nabla_{\bw}f_{n}(\bw)}{\tilde{\bw} - \bw} + \frac{1}{2}m \| \tilde{\bw} - \bw \|^{2}.
	\end{equation}
	where $\inprod{\bw}{\tilde{\bw}}$ is the Euclidean inner product between $\bw$ and $\tilde{\bw}$.
\item\namedlabel{probState:assump4}{A.3} Given an optimization algorithm that generates an approximate minimizer
	\[
	\bw_{n} \triangleq \mathcal{A}(\bw_{n-1},\{\bz_{n}(k)\}_{k=1}^{K_{n}})
	\]
	using $K_{n}$ samples $\{\bz_{n}(k)\}_{k=1}^{K_{n}}$, there exists a function $b(d_{0},K_{n})$ such that the following conditions hold:
	\begin{enumerate}
		\item If $K_{n}$ and $d_{0}$ are both $\mathcal{F}_{n-1}$-measurable random variables, it holds that
			\iftoggle{useTwoColumn}{
				\begin{align}
				&\| \bw_{n-1} - \bw_{n}^{*}\|^{2} \leq d_{0}^{2} \nonumber \\
				\label{probState:bBoundRand}
				&\;\;\;\;\;\;\; \Rightarrow \mathbb{E}[f_{n}(\bw_{n}) \;|\; \mathcal{F}_{n-1}] - f_{n}(\bw_{n}^{*}) \leq b(d_{0},K_{n}).
				\end{align}
			}{
				\begin{equation}
				\label{probState:bBoundRand}
				\| \bw_{n-1} - \bw_{n}^{*}\|^{2} \leq d_{0}^{2}  \;\;\Rightarrow\;\; \mathbb{E}[f_{n}(\bw_{n}) \;|\; \mathcal{F}_{n-1}] - f_{n}(\bw_{n}^{*}) \leq b(d_{0},K_{n}).
				\end{equation}
		}
		\item If $\tilde{K}_{n}$ and $\gamma$ are constants, it holds that
			\iftoggle{useTwoColumn}{
				\begin{align}
				&\mathbb{E}\| \bw_{n-1} - \bw_{n}^{*}\|^{2} \leq \gamma^{2} \nonumber \\
				\label{probState:bBoundDeterm}
				&\;\;\;\;\;\;\; \Rightarrow \mathbb{E}[f_{n}(\bw_{n})] - f_{n}(\bw_{n}^{*}) \leq b(\gamma,\tilde{K}_{n}).
				\end{align}
			}{
				\begin{equation}
				\label{probState:bBoundDeterm}
				\| \bw_{n-1} - \bw_{n}^{*}\|^{2} \leq \gamma^{2} \;\;\Rightarrow\;\; \mathbb{E}[f_{n}(\bw_{n}) \;|\; \mathcal{F}_{n-1}] - f_{n}(\bw_{n}^{*}) \leq b(\gamma,\tilde{K}_{n}).
				\end{equation}
		}
		\item The bound $b(d_{0},K_{n})$ is non-decreasing in $d_{0}$ and non-increasing in $K_{n}$.
	\end{enumerate}
	\item\namedlabel{probState:assump6}{A.4} Initial approximate minimizers $\bw_{1}$ and $\bw_{2}$ satisfy
	\[
	f_{i}(\bw_{i}) - f_{i}(\bw_{i}^{*}) \leq \epsilon_{i} \;\;\;\;\; i=1,2
	\]
	with $\epsilon_{1}$ and $\epsilon_{2}$ known.
\end{description}

\emph{Remarks:} For assumption \ref{probState:assump4}, we assume that the bound $b(d_{0},K_{n})$ depends on the number of samples $K_{n}$ and not the number of iterations. For the basic version of SGD, generally the number of iterations equals $K_{n}$, as each sample is used to produce a noisy gradient. See Appendix A of \cite{Wilson2016a} for a discussion of useful $b(d_{0},K_{n})$ bounds. For some bounds $b(d_{0},K)$, we may need to know parameters such as the strong convexity parameter. Estimating these parameters is discussed in Appendix C of \cite{Wilson2016a}. Finally, for assumption \ref{probState:assump6}, we can fix $K_{i}$ and set $\epsilon_{i} = b(\text{diam}(\mathcal{X}),K_{i})$ for $i=1,2$.

\subsection{Change in Minimizers Known}
\label{withRhoKnown}

% ~\eqref{probState:slowChangeConstDef} or
Following \cite{Wilson2016a}, we examine the case when the change in minimizers, $\rho$ in \eqref{probState:slowChangeDef}, is known. Suppose that $\epsilon_{n-1}$ bounds the excess risk at time $n-1$. Using the triangle inequality, strong convexity, Jensen's inequality, and \eqref{probState:slowChangeDef}, we have
\iftoggle{useTwoColumn}{
\begin{align}
\label{rhoKnown:basicDoBound}
\mathbb{E}\|\bw_{n-1} - \bw_{n}^{*}\|^{2} &\leq \left( \sqrt{\frac{2}{m}\epsilon_{n-1}} + \rho \right)^{2}
\end{align}	
}{
\begin{eqnarray}
\label{rhoKnown:basicDoBound}
\mathbb{E}\|\bw_{n-1} - \bw_{n}^{*}\|^{2} &\leq&  \left( \sqrt{\frac{2}{m}\epsilon_{n-1}} + \rho \right)^{2}
\end{eqnarray}
}
Now, by using the bound $b(d_{0},K_{n})$ from assumption~\ref{probState:assump1}, we set
\begin{eqnarray}
\label{withRhoKnown:epsNRecursion}
\epsilon_{n} &=& b\left( \sqrt{\frac{2 \epsilon_{n-1}}{m}} + \rho   , K_{n} \right) \;\;\; \forall n \geq 3
\end{eqnarray}
yielding a sequence of bounds on the excess risk. Note that this recursion only relies on the immediate past at time $n-1$ through $\epsilon_{n-1}$. To achieve $\epsilon_{n} \leq \epsilon$ for all $n$, we set
\[
K_{1} = \min\{ K \geq 1 \;|\; b\left( \text{diam}(\xSp), K  \right) \leq \epsilon \}
\]
and $K_{n} = K^{*}$ for $n \geq 2$ with
\begin{equation}
\label{withRhoKnown:KChoice}
K^{*} = \min\left\{ K \geq 1 \;\Bigg|\; b\left( \sqrt{\frac{2 \epsilon}{m}} + \rho , K  \right) \leq \epsilon \right\}
\end{equation}

In comparison, if we did not exploit the fact that the change is bounded by $\rho$, we would use the estimate $\text{diam}^{2}(\xSp)$ to bound $\mathbb{E}\|\bw_{n-1} - \bw_{n}^{*}\|^{2}$ and select $K_{n}$. If the bound in \eqref{rhoKnown:basicDoBound} is smaller than $\text{diam}^{2}(\xSp)$, then we would need significantly fewer samples $K_{n}$ to guarantee a desired excess risk.

\subsection{$K^{*}$ May Be Too Large}

In this section, we look at a case where $K^{*}$ can be too large. Suppose that $\rho = 0$, so the problems are not changing. In this case, we only need to take training samples at the first time instant and then we can stop taking samples, i.e., $K_{1} > 0$ and $K_{n} = 0$ for $n > 1$.

Suppose that $\epsilon_{1} \leq \epsilon$ and $\rho = 0$. In this case, from the analysis in the previous section, we pick
\[
K^{*} = \min\{ K \geq 1 \;|\; b\left( \sqrt{\frac{2 \epsilon}{m}}, K  \right) \leq \epsilon \}
\]
For an algorithm like SGD, the bound $b(d_{0},K)$ is roughly of the form (see \cite{Wilson2016a}):
\[
b(d_{0},K) \approx \frac{1}{K} + \frac{d_{0}^{2}}{K^{2}}
\]
The first term captures the asymptotic behavior of SGD and the second term accounts for the initial distance $d_{0}$. This form of $b(d_{0},K)$ implies that $K^{*} > 0$. However, by picking $K_{n} = 0$ for all $n \geq 2$, we could achieve $\epsilon_{n} = \epsilon_{1} \leq \epsilon$ for all $n \geq 2$.

This shows that the choice of $K^{*}$ is conservative and can be too large if the initial distance $d_0=0$. As a general rule, the choice of $K^{*}$ is useful if the term that depends on the initial distance, $d_{0}^{2}/K^{2}$, is comparable to the asymptotic term, $1/K$, in the $b(d_{0},K)$ bound.

\subsection{Estimating the Change in the Minimizers}
\label{estRho}

In practice, we do not know $\rho$, so we must construct an estimate $\hat{\rho}_{n}$ using the samples $\{\bz_{n}(k)\}_{k=1}^{K_{n}}$ from each distribution $p_{n}$. We introduce an approaches to estimate the one time step change, $\|\bw_{i}^{*} - \bw_{i-1}^{*}\|$, and methods to combine these estimates to produce an overall estimate of $\rho$. First, we work with the assumption that
\begin{equation}
\label{probState:slowChangeConstDef}
\|\bw_{i}^{*} - \bw_{i-1}^{*}\| = \rho
\end{equation}
as an intermediate step, and second, under assumption~\eqref{probState:slowChangeDef}. These estimates are from \cite{Wilson2016a}. For appropriately chosen sequences $\{t_{n}\}$ and for all $n$ large enough, we have $\hat{\rho}_{n} + t_{n} \geq \rho$ almost surely. With this property, analysis similar to that in Section~\ref{withRhoKnown} holds, which is provided in Section~\ref{withRhoUnknown}.

\subsubsection{Estimating One Step Change}
First, we develop an estimate $\tilde{\rho}_{i}$ of the one step changes $\|\bw_{i}^{*} - \bw_{i-1}^{*}\|$ using a method from \cite{Wilson2016a}. Implicitly, we assume that all one step estimates are bounded by $\text{diam}(\xSp)$, since trivially $\| \bw_{n}^{*} - \bw_{n-1}^{*}\| \leq \text{diam}(\xSp)$.

Using the triangle inequality and variational inequalities from \cite{Dontchev2009} yields
\begin{align}
\| \bw_{i}^{*} - \bw_{i-1}^{*} \| &\leq \| \bw_{i} - \bw_{i-1} \| + \|\bw_{i} - \bw_{i}^{*} \| + \| \bw_{i-1} - \bw_{i-1}^{*} \| \nonumber \\
&\leq \| \bw_{i} - \bw_{i-1} \| + \frac{1}{m} \| \nabla_{\bw} f_{i}(\bw_{i}) \| + \frac{1}{m} \| \nabla_{\bw} f_{i}(\bw_{i-1}) \| \nonumber
\end{align}
We then approximate $\| \nabla_{\bw} f_{i}(\bw_{i}) \| =  \| \mathbb{E}_{\bz_{i} \sim p_{i}} \left[ \nabla_{\bw} \ell(\bw_{i},\bz_{i})  \right] \|$
by a sample average approximation to yield the following estimate called the \emph{direct estimate}:
\iftoggle{useTwoColumn}{
\begin{align}
\tilde{\rho}_{i} &\triangleq \| \bw_{i} - \bw_{i-1} \| + \frac{1}{m} \Bigg\| \frac{1}{K_{i}} \sum_{k=1}^{K_{i}} \nabla_{\bw} \ell(\bw_{i},\bz_{i}(k))  \Bigg\| \nonumber \\
\label{dir_est_def}
&\qquad \qquad \qquad \qquad + \frac{1}{m} \Bigg\| \frac{1}{K_{i-1}} \sum_{k=1}^{K_{i-1}} \nabla_{\bw} \ell(\bw_{i-1},\bz_{i-1}(k))  \Bigg\|
\end{align}	
}{
\begin{align}
\label{dir_est_def}
\tilde{\rho}_{i} &\triangleq \| \bw_{i} - \bw_{i-1} \| + \frac{1}{m} \Bigg\| \frac{1}{K_{i}} \sum_{k=1}^{K_{i}} \nabla_{\bw} \ell(\bw_{i},\bz_{i}(k))  \Bigg\|  + \frac{1}{m} \Bigg\| \frac{1}{K_{i-1}} \sum_{k=1}^{K_{i-1}} \nabla_{\bw} \ell(\bw_{i-1},\bz_{i-1}(k))  \Bigg\|
\end{align}
}

\subsubsection{Combining One Step Estimates For Constant Change}
Assuming that $\| \bw_{i}^{*} - \bw_{i-1}^{*} \| = \rho$ from \eqref{probState:slowChangeConstDef}, we average the one step estimates $\tilde{\rho}_{i}$ to yield an overall estimate
\[
\hat{\rho}_{n} = \frac{1}{n-1} \sum_{i=2}^{n} \tilde{\rho}_{i}
\]
To proceed with our analysis, suppose that the following conditions hold:
\begin{description}
\item[\namedlabel{probState:assump3}{B.1}]  For each $n$, we can draw stochastic gradients $\sgrad{\bw}{\bz_{n}}{n}$ such that
\begin{equation}
\label{prob_form:noisy_grad_cons}
\mathbb{E}[\sgrad{\bw}{\bz_{n}}{n} ] = \nabla f_{n}(\bw).
\end{equation}
holds
\item[\namedlabel{probState:assump5}{B.2}] There exist constants $A,B \geq 0$ such that
\begin{equation}
\label{prob_form:L2_ngrad}
\mathbb{E}\left[ \| \sgrad{\bw}{\bz_{n}}{n} \|_{2}^{2} \right] \leq A + B \|\bw - \bw_{n}^{*}\|_{2}^{2}
\end{equation}
\item[\namedlabel{probState:assumpB1}{B.3}] There exist constants $C_{i}(K_{i})$ such that
\[
\mathbb{E}\left[ \|\bw_{i} - \tilde{\bw}_{i}\|^{2} \;|\; \mathcal{F}_{i-1} \right] \leq C_{i}^{2}(K_{i})
\]
\item[\namedlabel{probState:assumpB2}{B.4}] It holds that
\iftoggle{useTwoColumn}{
\begin{align}
\mathbb{E}&\left[ \|\sgrad{\bw}{\bz_{i}}{i} - \sgrad{\tilde{\bw}}{\bz_{i}}{i}\|^{2} \;|\; \mathcal{F}_{i-1}  \right] \leq M \| \bw - \tilde{\bw} \|^{2}
\end{align}
}{
\[
\mathbb{E}\left[ \|\sgrad{\bw}{\bz_{i}}{i} - \sgrad{\tilde{\bw}}{\bz_{i}}{i}\|^{2} \;|\; \mathcal{F}_{i-1}  \right] \leq M  \| \bw - \tilde{\bw} \|^{2}
\]
}
and
\[
\mathbb{E}\left[ \|\sgrad{\bw}{\bz_{i}}{i} - \nabla f_{i}(\bw) \|^{2} \;|\; \mathcal{F}_{i-1}  \right] \leq \sigma^{2}
\]
\item[\namedlabel{probState:assumpB3}{B.5}] The gradients are bounded in the sense that
\[
\| \sgrad{\bw}{\bz}{n} \| \leq G \;\;\;\; \forall \bw \in \xSp , \bz \in \mathcal{Z}
\]
\end{description}

Assumption~\ref{probState:assump3} guarantees that the gradients are unbiased. Assumption~\ref{probState:assump5} controls how fast the gradients grow as we move away from the minimizer $\bw_{n}^{*}$.
Assumption~\ref{probState:assumpB1} controls how far apart two independent outputs of the optimization algorithm $\bw_{i}$ and $\tilde{\bw}_{i}$ are, starting from $\bw_{i-1}$. Assumption~\ref{probState:assumpB2} controls how the gradient grows for two pairs $(\bw,\bz_{i})$ and $(\tilde{\bw},\bz_{i})$. Finally, assumption~\ref{probState:assumpB3} is reasonable if the space $\mathcal{Z}$ that contains the $\bz_{n}$ has finite diameter and the gradients of the lost function are continuous jointly in $(\bw,\bz)$. In this case, it holds that
\[
\sup_{\bw \in \xSp,\bz \in \mathcal{Z}} \| \sgrad{\bw}{\bz}{n}\| < \infty
\]

Theorem~\ref{rho_conc_eq} from \cite{Wilson2016a} guarantees that the direct estimate from \eqref{dir_est_def} bounds $\rho$.
\begin{thm}
\label{rho_conc_eq}
Provided that \ref{probState:assumpB1}-\ref{probState:assumpB3} hold and our sequence $\{t_{n}\}$\footnote{Note that a choice of $t_{n}$ that is no greater than $1/\sqrt{n-1}$ works here.} satisfies
\[
\sum_{n=2}^{\infty} \left( \exp\left\{ - \frac{(n-1)t_{n}^{2}}{18 \text{diam}^{2}(\xSp)}\right\} + 2\exp\left\{ - \frac{m^2(n-1)t_{n}^{2}}{72 G^{2}} \right\} \right) < +\infty
\]
it holds that for all $n$ large enough
\[
\hat{\rho}_{n} + D_{n} + t_{n} \geq \rho
\]
almost surely with
\iftoggle{useTwoColumn}{
	$D_{n}$ defined in \eqref{rho_conc_eq:Dn}
	\begin{figure*}[!t]
		\normalsize
		\begin{equation}
		\label{rho_conc_eq:Dn}
		D_{n} = \frac{1}{n-1} \left[ \left(1 + \frac{M}{m}\right) C_{1} + \sqrt{\frac{\sigma}{K_{1}}}  + 2 \sum_{i=2}^{n-1} \left( \left(1 + \frac{M}{m}\right) C_{i} + \sqrt{\frac{\sigma}{K_{i}}} \right) + \left(1 + \frac{M}{m}\right) C_{n} + \sqrt{\frac{\sigma}{K_{n}}}  \right]
		\end{equation}
		\hrulefill
		\vspace*{4pt}
	\end{figure*}
}{
\begin{equation}
\label{rho_conc_eq:Dn}
D_{n} = \frac{1}{n-1} \left[ \left(1 + \frac{M}{m}\right) C_{1} + \sqrt{\frac{\sigma}{K_{1}}}  + 2 \sum_{i=2}^{n-1} \left( \left(1 + \frac{M}{m}\right) C_{i} + \sqrt{\frac{\sigma}{K_{i}}} \right) + \left(1 + \frac{M}{m}\right) C_{n} + \sqrt{\frac{\sigma}{K_{n}}}  \right]
\end{equation}
}
\end{thm}
\begin{proof}
	See \cite{Wilson2016a}.
\end{proof}

\subsubsection{Combining One Step Estimates For Bounded Change}
We now look at estimating $\rho$ in the case that $\| \bw_{n}^{*} - \bw_{n-1}^{*} \| \leq \rho$. We set
\[
\rho_{i} \triangleq \| \bw_{i}^{*} - \bw_{i-1}^{*} \|
\]
Although, it may seem natural to combine the estimates using
\begin{equation}
\label{estRho:maxCombBad}
\hat{\rho}_{n} = \max\{\tilde{\rho}_{2},\ldots,\tilde{\rho}_{n}\}
\end{equation}
this method has a serious drawback. Since $\{\tilde{\rho}_{i}\}$ are random variables, if we combine them by taking their maximum, any particular one step estimate $\tilde{\rho}_{i}$ that is large will pull up the overall estimate $\hat{\rho}_{n}$. This would drive $\hat{\rho}_{n} \to \textrm{diam}(\xSp)$, as $n \to \infty$, resulting in a $\hat{\rho}_{n}$ that is trivial in the limit of large $n$.

We introduce an estimate from \cite{Wilson2016a} that overcomes this defect. We need the following assumptions:
\begin{description}
	\item[\namedlabel{probState:assumpB4}{B.4}] We have estimates $\hat{h}_{W}: \mathbb{R}^{W} \to \mathbb{R}$ that are non-decreasing in their arguments such that
	\[
	\mathbb{E}[ \hat{h}_{W}(\rho_{j},\ldots,\rho_{j-W+1}) ] \geq \rho
	\]
	\item[\namedlabel{probState:assumpB5}{B.5}] There exists absolute constants $\{b_{i}\}_{i=1}^{W}$ for any fixed $W$ such that $\forall \bm{p},\bm{q} \in \mathbb{R}^{W}_{\geq 0}$
	\[
	|\hat{h}_{W}(p_{1},\ldots,p_{W}) - \hat{h}_{W}(q_{1},\ldots,q_{W})| \leq \sum_{i=1}^{W} b_{i} |p_{i} - q_{i}|
	\]
\end{description}

For example, if $\rho_{i} \overset{\text{iid}}{\sim} \text{Unif}[0,\rho]$, then
\[
\hat{h}_{W}\left( \rho_{i},\rho_{i+1},\ldots,\rho_{i+W-1} \right) = \frac{W+1}{W} \max\{ \rho_{i},\rho_{i+1},\ldots,\rho_{i+W-1} \}
\]
is an estimator of $\rho$ with the required properties. Also, note that the two conditions on the estimator in~\ref{probState:assumpB3} imply that
\[
\mathbb{E}\left[ \hat{h}_{W}(\tilde{\rho}_{j},\ldots,\tilde{\rho}_{j-W+1}) \right] \geq \mathbb{E}\left[ \hat{h}_{W}(\rho_{j},\ldots,\rho_{j-W+1}) \right] \geq \rho
\]

Given an estimator satisfying assumption~\ref{probState:assumpB3}, we compute
\[
\tilde{\rho}^{(i)} = \hat{h}_{W}(\tilde{\rho}_{i},\tilde{\rho}_{i-1},\ldots,\tilde{\rho}_{i-W+1})
\]
and set
\begin{eqnarray}
\hat{\rho}_{n} &=& \frac{1}{n-1} \sum_{i=2}^{n} \tilde{\rho}^{(i)} \nonumber \\
\label{ineqCond:basicEst}
&=&\!\!\!\!\!\!\!\! \frac{1}{n-1} \sum_{i=2}^{n} \hat{h}_{\min\{W,i-1\}}(\tilde{\rho}_{i},\tilde{\rho}_{i-1},\ldots,\tilde{\rho}_{\max\{i-W+1,2\}})
\end{eqnarray}

Under assumptions~\ref{probState:assumpB1}-\ref{probState:assumpB5}, we can then show that
\begin{equation}
\label{dir_est_comb_ineq}
\hat{\rho}_{n} = \frac{1}{n-W} \sum_{i=W+1}^{n} \bar{\rho}^{(i)}
\end{equation}
eventually upper bounds $\rho$, as stated in the following theorem.

\setcounter{dir_est_ineq_cnt}{\value{thm}}
\begin{thm}
\label{rho_conc_ineq}
Provided that \ref{probState:assumpB1}-\ref{probState:assumpB5} hold and our sequence $\{t_{n}\}$ satisfies
\iftoggle{useTwoColumn}{
\begin{align}
\sum_{n=2}^{\infty} &\left( \exp\left\{ - \frac{(n-W)^{2}t_{n}^{2}}{18(n-1) \text{diam}^{2}(\xSp)\left( \sum_{j=1}^{W} b_{j} \right)^{2}}\right\} \right. \nonumber \\
&\qquad \left. + 2\exp\left\{ - \frac{m^2(n-W)^2 t_{n}^{2}}{72 (n-1) G^{2}\left( \sum_{j=1}^{W} b_{j} \right)^{2}} \right\} \right) < +\infty
\end{align}
}{
\[
\sum_{n=2}^{\infty} \left( \exp\left\{ - \frac{(n-W)^{2}t_{n}^{2}}{18(n-1) \text{diam}^{2}(\xSp)\left( \sum_{j=1}^{W} b_{j} \right)^{2}}\right\} + 2\exp\left\{ - \frac{m^2(n-W)^2 t_{n}^{2}}{72 (n-1) G^{2}\left( \sum_{j=1}^{W} b_{j} \right)^{2}} \right\} \right) < +\infty
\]
}
it holds that for all $n$ large enough
\[
\hat{\rho}_{n} + \left( \frac{n-1}{n-W} \sum_{j = 1}^{W} b_{j} \right) D_{n} + t_{n} \geq \rho
\]
with $D_{n}$ from Theorem~\ref{rho_conc_eq}.
\end{thm}
\begin{proof}
See \cite{Wilson2016a}.
\end{proof}

\subsection{Change in Minimizers Unknown}
\label{withRhoUnknown}
We now present an extension of the results in Section~\ref{withRhoKnown}, obtained by replacing $\rho$ with its estimate given in Section~\ref{estRho}. Our analysis depends on the following crucial assumption:
\begin{description}
\item[\namedlabel{probState:assumpC1}{C.1}] For appropriate sequences $\{t_{n}\}$, for all $n$ sufficiently large it holds that $\hat{\rho}_{n} + t_{n} \geq \rho$ almost surely.
\item[\namedlabel{probState:assumpC2}{C.2}] $b(d_{0},K_{n})$ factors as $b(d_{0},K_{n}) = \alpha(K_{n}) d_{0}^{2} + \beta(K_{n})$
\end{description}
We have demonstrated that assumption~\ref{probState:assumpC1} holds for the direct estimate of $\rho$ under \eqref{probState:slowChangeConstDef} and \eqref{probState:slowChangeDef}. Note that whether we assume \eqref{probState:slowChangeConstDef} or \eqref{probState:slowChangeDef} does not matter for analysis. We start with a general result showing that for appropriate choices of $K_{n}$, we control the excess risk.
\begin{thm}
\label{withRhoUnknown:meanGapRhoKnownLemma}
Under assumptions~\ref{probState:assumpC1}-~\ref{probState:assumpC2}, with $K_{n} \geq K^{*}$ for all $n$ large enough, where $K^{*}$ is defined in \eqref{withRhoKnown:KChoice}, we have
\begin{equation}
\label{withRhoUnknown:keyERControl}
\limsup_{n \to \infty} \left( \mathbb{E}[f_{n}(\bw_{n})] - f_{n}(\bw_{n}^{*})  \right) \leq \epsilon
\end{equation}
almost surely
\end{thm}
\begin{proof}
See \cite{Wilson2016a}.
\end{proof}
This theorem shows that for any choice of samples $K_{n}$ such that $K_{n} \geq K^{*}$ for $n$ large enough, it follows that the excess risk can be controlled in the sense of \eqref{withRhoUnknown:keyERControl}.

\subsubsection{Update Past Excess Risk Bounds}
\label{withRhoUnknown:updatePast}

We first consider updating all past excess risk bounds as we go. At time $n$, we plug-in $\hat{\rho}_{n-1} + t_{n-1}$ in place of $\rho$ and follow the analysis of Section~\ref{withRhoKnown}. Define for $i=1,\ldots,n$
\begin{eqnarray}
\label{withRhoUnknown:epsEstUpdate}
\hat{\epsilon}_{i}^{(n)} &=& b\left( \left(  \sqrt{\frac{2}{m} \hat{\epsilon}_{i-1}^{(n)}}  + (\hat{\rho}_{n-1} + t_{n-1}) \right)^{2},K_{i} \right)
\end{eqnarray}
If it holds that $\hat{\rho}_{n-1} + t_{n-1} \geq \rho$, then ${\mathbb{E}\left[ f_{n}(\bw_{n})  \right] - f_{n}(\bw_{n}^{*}) \leq \hat{\epsilon}_{n}^{(i)}}$ for ${i=1,\ldots,n}$. Assumption~\ref{probState:assumpC1} guarantees that this holds for all $n$ large enough almost surely. We can thus set $K_{n}$ equal to the smallest $K$ such that
\iftoggle{useTwoColumn} {
\begin{align}
&K_{n} = \min\left\{ K \geq 1 \;\Bigg|\; b\left( \left( \sqrt{ \frac{2}{m} \max\{\hat{\epsilon}^{(n-1)}_{n-1},\epsilon\} } + \right. \right. \right. \nonumber \\
\label{withRhoUnknown:KnChoiceUpdate}
&\qquad \qquad \qquad \qquad \qquad \qquad \left. \left. \left. ( \hat{\rho}_{n-1} + t_{n-1} ) \right)^{2} , K  \right) \leq \epsilon \right\}
\end{align}	
}{
\begin{equation}
\label{withRhoUnknown:KnChoiceUpdate}
K_{n} = \min\left\{ K \geq 1 \;\Bigg|\; b\left( \left( \sqrt{ \frac{2}{m} \max\{\hat{\epsilon}^{(n-1)}_{n-1},\epsilon\} } + ( \hat{\rho}_{n-1} + t_{n-1} ) \right)^{2} , K  \right) \leq \epsilon \right\}
\end{equation}
}
for all $n \geq 3$ to achieve excess risk $\epsilon$. The maximum in this definition ensures that when $\hat{\rho}_{n-1} + t_{n-1} \geq \rho$, $K_{n} \geq K^{*}$ with $K^{*}$ from \eqref{withRhoKnown:KChoice}. We can therefore apply Theorem~\ref{withRhoUnknown:meanGapRhoKnownLemma}.

\subsubsection{Do Not Update Past Excess Risk Bounds}
\label{withRhoUnknown:doNotUpdatePast}

Updating all past estimates of the excess risk bounds from time $1$ up to $n$ imposes a computational and memory  burden. Suppose that for all $n \geq 3$ we set
\begin{equation}
\label{withRhoUnknown:KnChoice}
K_{n} = \min \left\{ K \geq 1 \;\Bigg|\; b\left( \left( \sqrt{\frac{2\epsilon}{m}} + ( \hat{\rho}_{n-1} + t_{n-1} ) \right)^{2} , K  \right) \leq \epsilon \right\}
\end{equation}
This is the same form as the choice in \eqref{withRhoKnown:KChoice} with $\hat{\rho}_{n-1} + t_{n-1}$ in place of $\rho$. Due to assumption~\ref{probState:assumpC1}, for all $n$ large enough it holds that $\hat{\rho}_{n} + t_{n} \geq \rho$ almost surely. Then by the monotonicity assumption in \ref{probState:assump1}, for all $n$ large enough we pick $K_{n} \geq K^{*}$ almost surely. We can therefore apply Theorem~\ref{withRhoUnknown:meanGapRhoKnownLemma}.

\section{Bound on $\rho$-Estimate Overshoot}
\label{rhoBoundOvershoot}

Since we assume that the solution space $\xSp$ has bounded diameter, we always have the trivial bound
\[
\| \bw_{n}^{*} - \bw_{n-1}^{*} \|_{2} \leq \textrm{diam}(\xSp)
\]
An estimate of the change in minimizers, $\hat{\rho}_{n}$, is only interesting if the bound is non-trivial, i.e., $\hat{\rho}_{n} < \textrm{diam}(\xSp)$ when $\rho < \textrm{diam}(\xSp)$. In prior work \cite{Wilson2016a}, we have proved the for sufficiently large $n$, $\hat{\rho}_{n} + t_{n} \geq \rho$ almost surely. In this section, we look at proving an upper bound on how much $\hat{\rho}_{n}$ can overshoot $\rho$ to show that this estimate is non-trivial.

When we proved that $\hat{\rho}_{n}$ eventually upper bounds $\rho$, we did not use the fact that the points $\bw_{n}$ at which we are evaluating the one-step estimates  are approximate minimizers. In particular, that proof would still hold even if we selected the $\bw_{n}$ randomly from the solution space $\xSp$ without using the samples $\{\bz_{n}(k)\}_{k=1}^{K_{n}}$ at all. In contrast, controlling the overshoot depends critically on the fact that the points at which we evaluate the one-step estimates are approximate minimizers. The solution quality of the approximate minimizers measured by $\epsilon$ in \eqref{probState:meanGapDef} will control the size of the overshoot, as seen in the following theorem.

\begin{thm}
Suppose that the following conditions hold:
\begin{enumerate}
	\item The sequence of excess risks achieved, $\epsilon_{i}$, $i=1,2, \ldots$, satisfies
	\[
	\limsup_{n \to \infty} \epsilon_{n} \leq \epsilon
	\]
	\item The loss function $f_{n}(\bw)$ has Lipschitz continuous gradients with parameter $M$, i.e.,
	\[
	f_{n}(\bw) \leq f_{n}(\tilde{\bw}) + \inprod{\nabla f_{n}(\tilde{\bw})}{\tilde{\bw} - \bw} + \frac{1}{2} M \| \tilde{\bw} - \bw \|^{2}
	\]
	\item For all $i$ large enough, we have that $K_{i} \geq \tilde{K}$ for a constant $\tilde{K}$.
\end{enumerate}
Then it follows that
\[
\limsup_{n \to \infty} \mathbb{E}[\hat{\rho}_{n}] \leq \rho + \frac{2M}{m^{3/2}} \epsilon + G
\]
where
\[
G \triangleq \frac{2M}{m}  C(\tilde{K}) + \frac{1}{m} \left( \frac{\sigma}{\tilde{K}} \right)^{1/2}
\]
\end{thm}
\begin{proof}
First, we look at the one step estimates. It holds that
\begin{align}
\tilde{\rho}_{i} - \rho &= \| \bw_{i} - \bw_{i-1}\| - \| \bw_{i}^{*} - \bw_{i-1}^{*}\| + \frac{1}{m} \| \hat{G}_{i} \| + \frac{1}{m} \| \hat{G}_{i-1} \|  \nonumber \\
&\leq \left|  \| \bw_{i} - \bw_{i-1}\| - \| \bw_{i}^{*} - \bw_{i-1}^{*}\| \right| + \frac{1}{m} \| \hat{G}_{i} \| + \frac{1}{m} \| \hat{G}_{i-1} \|  \nonumber \\
&\leq \frac{1}{m} \| \nabla f_{i}(\bw_{i}) \| + \frac{1}{m} \| \nabla f_{i-1}(\bw_{i-1}) \| + \frac{1}{m} \| \hat{G}_{i} \| + \frac{1}{m} \| \hat{G}_{i-1} \|  \nonumber \\
&\leq \frac{2}{m} \| \nabla f_{i}(\bw_{i}) \| + \frac{2}{m} \| \nabla f_{i-1}(\bw_{i-1}) \| \nonumber \\
&\;\;\;\;\;\; + \frac{1}{m} \| \nabla f_{i}(\bw_{i}) - \hat{G}_{i} \| + \frac{1}{m} \| \nabla f_{i-1}(\bw_{i-1})- \hat{G}_{i-1} \|.  \nonumber
\end{align}
By the Lipschitz gradient assumption, we have
\[
\| \nabla f_{i}(\bw) \| \leq M \| \bw - \bw_{i}^{*} \|.
\]
Then it follows by strong convexity that
\[
\| \bw - \bw_{i}^{*} \| \leq \sqrt{\frac{2}{m} (f_{i}(\bw) - f_{i}(\bw_{i}^{*}))}
\]
and therefore we have
\[
\| \nabla f_{i}(\bw) \| \leq \frac{\sqrt{2}M}{\sqrt{m}} \sqrt{f_{i}(\bw) - f_{i}(\bw_{i}^{*})}.
\]
Since the square-root is concave, by Jensen's inequality, we have
\[
\mathbb{E}[\| \nabla f_{i}(\bw_{i}) \|] \leq \frac{\sqrt{2}M}{\sqrt{m}} \epsilon_{i}.
\]
This in turn implies that
\begin{align}
\mathbb{E}&[\tilde{\rho}_{i}] - \rho \nonumber \\
& \leq \frac{2\sqrt{2}M}{m^{3/2}} \epsilon_{i} + \frac{2\sqrt{2}M}{m^{3/2}} \epsilon_{i-1} + \nonumber \\
& \;\;\;\;\;\;\; \frac{1}{m} \mathbb{E} \| \nabla f_{i}(\bw_{i}) - \hat{G}_{i} \| + \frac{1}{m} \mathbb{E}\| \nabla f_{i-1}(\bw_{i-1})- \hat{G}_{i-1} \|. \nonumber
\end{align}
Next, we look at bounding $\mathbb{E} \| \nabla f_{i}(\bw_{i}) - \hat{G}_{i} \|$. Define
\[
\tilde{G}_{i} = \frac{1}{K_{i}} \sum_{k=1}^{K_{i}} \nabla_{\bw} \ell(\tilde{\bw}_{i},\bz_{i}(k)).
\]
Then we have
\begin{align}
&\| \nabla f_{i}(\bw_{i}) - \hat{G}_{i} \| \nonumber \\
&\;\;\; \leq \|  \hat{G}_{i} - \tilde{G}_{i} \| + \|  \tilde{G}_{i} - \nabla f_{i}(\tilde{\bw}_{i}) \| + \| \nabla f_{i}(\tilde{\bw}_{i}) - \nabla f_{i}(\bw_{i}) \| \nonumber \\
&\;\;\; \leq \|  \hat{G}_{i} - \tilde{G}_{i} \| + \|  \tilde{G}_{i} - \nabla f_{i}(\tilde{\bw}_{i}) \| + M \| \tilde{\bw}_{i} - \bw_{i} \|. \nonumber
\end{align}
Using the direct estimate lower bound analysis from \cite{Wilson2016a} it follows that
\begin{eqnarray}
\mathbb{E} \left[ \| \nabla f_{i}(\bw_{i}) - \hat{G}_{i} \| \;|\; \mathcal{F}_{i-1} \right] &\leq& M C_{i} + \left( \frac{\sigma}{K_{i}} \right)^{1/2} + M  C_{i}.
\end{eqnarray}
This shows that
\begin{align}
\mathbb{E}&[\tilde{\rho}_{i}] - \rho \nonumber \\
&\leq \frac{2\sqrt{2}M}{m^{3/2}} \epsilon_{i} + \frac{2\sqrt{2}M}{m^{3/2}} \epsilon_{i-1} +  \mathbb{E}\left[ \frac{2M}{m}  C_{i} + \frac{1}{m} \left( \frac{\sigma}{K_{i}} \right)^{1/2} \right] \nonumber \\
& + \mathbb{E} \left[ \frac{2M}{m}  C_{i} +  \frac{M}{m}  C_{i-1} + \frac{1}{m} \left( \frac{\sigma}{K_{i-1}} \right)^{1/2}  \right]
\end{align}
Then plugging in the definition of $\hat{\rho}_{n}$ it follows that
\begin{eqnarray}
\limsup_{n \to \infty} \mathbb{E}[\hat{\rho}_{n}] &\leq& \rho + \frac{2\sqrt{2}M}{m^{3/2}} \epsilon +\frac{2M}{m}  C(\tilde{K}) + \frac{1}{m} \left( \frac{\sigma}{\tilde{K}} \right)^{1/2} \nonumber \\
&=& \rho + \frac{2\sqrt{2}M}{m^{3/2}} \epsilon  + G \label{eq:limsup}
\end{eqnarray}
\end{proof}

This shows that the direct estimate is a non-trivial upper bound for sufficiently small $\epsilon$. Note that in practice, the $\tilde{K}$ will be a function of $\epsilon$, since we can pick $\tilde{K}= K^{*}$ with $K^{*}$ defined in \eqref{withRhoKnown:KChoice}. Note that $K^{*}$ is itself a function of $\epsilon$. This means the $G$ term in \eqref{eq:limsup}, which is a function of $\tilde{K}$ is also a function of $\epsilon$. Thus the entire overshoot term is a function of $\epsilon$, and in fact by inspection, it goes to zero as $\epsilon \to 0$ if $K^{*} \to \infty$ as $\epsilon \to 0$ (as  $K^{*} $ defined in \eqref{withRhoKnown:KChoice} does).

\section{Extensions Relevant to Machine Learning Applications}
\label{ext}

\subsection{Cost Approach}
\label{ext:cost}

A natural way to assess the usefulness of our approach is to choose a number of samples $\{K_{n}\}_{n=1}^{T}$ over a horizon of length $T$ using the choice in \eqref{withRhoUnknown:KnChoiceUpdate} and \eqref{withRhoUnknown:KnChoice}, and compare against taking
\[
\sum_{n=1}^{T} K_{n}
\]
samples at time $n=1$ and no samples at the other $T-1$ time instants. See Section~\ref{exper} for such a comparison.

In this section, we consider a different type of comparison based on assuming that there is a cost $p(K_{n})$ of taking $K_{n}$ samples. For example, we could have
\begin{equation}
\label{ext:cost:costExample}
p(K) = P_{0} \mathbbm{1}_{\{K > 0 \}} + P_{1} K
\end{equation}
This implies we pay a fixed cost of $P_{0}$ any time we take at least one sample and a marginal cost of $P_{1}$ per sample. We want to control the excess risk by deciding when to take samples, and how many samples to take with a total budget $P$ over a horizon of length $T$, i.e.,
\begin{equation}
\label{ext:cost:budget}
\sum_{n=1}^{T} p(K_{n}) \leq P
\end{equation}
For the option of taking all samples up front:
\begin{equation}
\label{ext:cost:costUpFront}
K_{n} = \begin{cases}
\max\left\{ K \geq 1 \;\big|\; p(K) \leq P \right\}, & n = 1 \\
0, & 2 \leq n \leq T
\end{cases}
\end{equation}
Another option is to sample every $\Delta T$ time instants and divide the cost budget evenly over the times that we take samples using
\begin{equation}
\label{ext:cost:costPeriodic}
K_{n} = \begin{cases}
\max\left\{ K \geq 1 \;\big|\; p(K) \leq \Big\lfloor \frac{P}{T/\Delta T } \Big \rfloor \right\}, &\text{if }\Delta T \text{ divides } (n - 1)  \\
0, & \text{else}.
\end{cases}
\end{equation}

For analysis, we need Assumption~\ref{probState:assumpC1} and the following additional assumptions:
\begin{description}
	\item[\namedlabel{probState:assumpE1}{D.1}] There exists a function $e(\|\bw - \bw_{n}^{*}\|_{2}^{2})$ such that
	\[
	f_{n}(\bw) - f_{n}(\bw_{n}^{*}) \leq e(\|\bw - \bw_{n}^{*}\|_{2}^{2})
	\]
\end{description}
For example, suppose that the functions $f_{n}(\bw)$ have Lipschitz continuous gradients with modulus $M$ and $\bw_{n}^{*} \in \text{int}(\xSp)$ for all $n \geq 1$, where $\text{int}(\xSp)$ is the interior of $\xSp$. By the descent lemma \cite{Bertsekas1999}, we have
\begin{eqnarray}
f_{n}(\bw_{n}) - f_{n}(\bw_{n}^{*}) &\leq& \inprod{\nabla f_{n}(\bw_{n}^{*})}{\bw_{n} - \bw_{n}^{*}} + \frac{1}{2} M \|\bw_{n} - \bw_{n}^{*}\|_{2}^{2} \nonumber \\
&=&  \frac{1}{2} M \|\bw_{n} - \bw_{n}^{*}\|_{2}^{2} \nonumber
\end{eqnarray}
Thus, we can set
\[
e(\|\bw_{n} - \bw_{n}^{*}\|_{2}) = \frac{1}{2} M \|\bw_{n} - \bw_{n}^{*}\|_{2}^{2}
\]

Since we need to consider the possibility that $K_{n} = 0$ for some $n$ in $\{1,\ldots,T\}$ but still provide estimates of the excess risk, we need an alternate version of the bound in \eqref{withRhoUnknown:epsEstUpdate}. Define
\[
t_{s}(n) = \max\left\{ m \;|\; 1 \leq m \leq n \;\text{and}\; K_{m} > 0  \right\}
\]
where $t_{s}(n)$ is the last time no later than $n$ at which samples were taken. If no samples have been taken so far, then by convention $t_{s}(n) = +\infty$. We construct the recursively defined function $\tilde{b}_{n}(\rho,K_{n})$ by considering the following four cases:
\begin{enumerate}
	\item No samples have been taken by time $n$:
	\[
	\tilde{b}_{n}(\rho,K_{n}) \triangleq e(\text{diam}(\xSp))
	\]
	\item Samples taken at time $n$ for the first time
	\[
	\tilde{b}_{n}(\rho,K_{n}) \triangleq b(\text{diam}(\xSp),K_{n})
	\]
	\item No samples taken at time $n$ but samples have been taken previously
	\[
	\tilde{b}_{n}(\rho,K_{n}) \triangleq e\left( \sqrt{\frac{2}{m} \tilde{b}_{t_{s}(n-1)}} + \left( (n - t_{s}(n-1)) \rho  \right) \right)
	\]
	\item Samples taken at time $n$ and samples have been taken previously
	\[
	\tilde{b}_{n}(\rho,K_{n}) \triangleq b\left( \sqrt{ \frac{4}{m} \tilde{b}_{t_{s}(n-1)} + 2 \left( (n - t_{s}(n-1)) \rho  \right)^{2} } , K_{n} \right)
	\]
	where $\tilde{b}_{t_{s}(n-1)}$ is the bound on the excess risk at time $t_{s}(n-1)$.
\end{enumerate}

Suppose that over a time horizon of length $T$ we have a total cost budget $P$ with respect to the number of samples $\{K_{n}\}_{n=1}^{T}$ as in \eqref{ext:cost:budget}. Define the \emph{excess risk gaps}
\[
\xi_{n} = \left( \tilde{b}_{n}\left(\rho,K_{n}\right) - \epsilon \right)_{+}
\]
with $(x)_{+} = \max\{x,0\}$. The variable $\xi_{n}$ is the extent to which the target excess risk of $\epsilon$ is violated upwards. If our excess risk is below our target level $\epsilon$, then we set $\xi_{n}=0$. Our goal is to minimize the size of the $\xi_{n}$, while taking into account the cost constraint in \eqref{ext:cost:budget}. To control the size of $\xi_{n}$, suppose that we have a function $\phi:\mathbb{R}^{T} \to \mathbb{R}$ that describes the cumulative loss of the excess risk gaps $\xi_{1},\ldots,\xi_{T}$.

We now provide some possible choices for $\phi(\xi_{1},\ldots,\xi_{T})$:
\begin{equation} \label{eq:phi1}
\phi(\xi_{1},\ldots,\xi_{T}) = \frac{1}{T} \sum_{n=1}^{T} \xi_{t}
\end{equation}
\begin{equation} \label{eq:phi2}
\phi(\xi_{1},\ldots,\xi_{T}) = \max\{\xi_{1},\ldots,\xi_{T}\}
\end{equation}
\begin{equation} \label{eq:phi3}
	\phi(\xi_{1},\ldots,\xi_{T}) = \max_{(a,b) \in \bm{\tau}} \sum_{n=a}^{b} \xi_{n}
\end{equation}
	with
	\[
	\bm{\tau} = \left\{ (a,b) \;|\; a < b \;,\; \xi_{a} \leq \xi_{a+1} \leq \cdots \leq \xi_{b} \right\}.
	\]
The choices given in \eqref{eq:phi1} and  \eqref{eq:phi2} penalize the average and maximum excess risk gaps, respectively. In practice, with these choices, we will stop taking samples before the horizon $T$ resulting in relatively poor performance towards the end of the horizon. The third choice gets around this problem by penalizing large increasing runs of excess risk gaps, and tends to favor a more uniform choice of the number of samples $K_{n}$.

We first consider the case when $\rho$ is known to us and plan over the horizon of length $T$ by solving the following optimization problem:
\begin{equation}
\label{ext:cost:rhoKnownHorizon}
\arraycolsep=1.4pt\def\arraystretch{1.5}
\begin{array}{ll@{}ll}
\underset{K_{1},\ldots,K_{T}}{\text{minimize}}  & \displaystyle \phi(\xi_{1},\ldots,\xi_{T}) &\\
\text{subject to}& \displaystyle \sum_{n=1}^{T} p(\rho,K_{n}) \leq P \\
& \displaystyle \mathbbm{1}_{\{K_{1} > 0\}} \leq \mathbbm{1}_{\{K_{2} > 0\}} & \\
& \displaystyle \mathbbm{1}_{\{K_{n} > 0\}} \leq \mathbbm{1}_{\{K_{n-1} > 0\}} + \mathbbm{1}_{\{K_{n+1} > 0\}} & \;\;\; n=2,\ldots,T-1 \\
& \displaystyle \mathbbm{1}_{\{K_{T-1} > 0\}} \leq \mathbbm{1}_{\{K_{T} > 0\}} & \\
& K_{n} \in \mathbb{Z}_{\geq 0}  &\;\;\; n=1,\ldots,T
\end{array}
\end{equation}
The idea of this problem is to satisfy the excess risk bound $\epsilon$ with minimal violation $\phi(\xi_{1},\ldots,\xi_{T})$.

To estimate $\rho$, we need samples from consecutive time instants. Therefore, we impose the constraint that if we take samples at time $n$, then we must take samples at either time $n-1$ or time $n+1$ through the constraint
\[
\mathbbm{1}_{\{K_{n} > 0\}} \leq \mathbbm{1}_{\{K_{n-1} > 0\}} + \mathbbm{1}_{\{K_{n+1} > 0\}}
\]
The problem in \eqref{ext:cost:rhoKnownHorizon} is a mixed integer non-linear programming problem (MINLP). There are no general methods to efficiently solve this MINLP, and we therefore consider a relaxation of this problem later.

In the case that we know $\rho$, we can plan the number of samples ahead of time before any samples have been taken. When $\rho$ is unknown, we cannot plan over the entire horizon. Instead, at each time instant $m$ we have to plan over the remaining time horizon of length $T - m + 1$, while using the estimate $\hat{\rho}_{m-1} + t_{m-1}$ in place of $\rho$ and the remaining cost budget
\[
P - \sum_{n=1}^{m-1} p(K_{n}).
\]
We then consider the cost-to-go problem
\begin{equation}
\label{ext:cost:rhoUnknownPartial}
\arraycolsep=1.4pt\def\arraystretch{1.5}
\begin{array}{ll@{}ll}
\underset{K_{m},\ldots,K_{T}}{\text{minimize}}  & \displaystyle\phi(\xi_{m},\ldots,\xi_{T}) &\\
\text{subject to}& \displaystyle \sum_{n=m}^{T} p(K_{n}) \leq P - \sum_{n=1}^{m-1} p(K_{n}) \\
& \displaystyle \mathbbm{1}_{\{K_{m} > 0\}} \leq \mathbbm{1}_{\{K_{m+1} > 0\}} & \\
& \displaystyle \mathbbm{1}_{\{K_{m} > 0\}} \leq \mathbbm{1}_{\{K_{m-1} > 0\}} + \mathbbm{1}_{\{K_{m+1} > 0\}} & \;\;\; n=m+1,\ldots,T-1 \\
& \displaystyle \mathbbm{1}_{\{K_{T-1} > 0\}} \leq \mathbbm{1}_{\{K_{T} > 0\}} & \\
& K_{n} \in \mathbb{Z}_{\geq 0}  &\;\;\; n=m,\ldots,T
\end{array}
\end{equation}
This is the same form as \eqref{ext:cost:rhoKnownHorizon}, except that it is over the time horizon from $n=m,\ldots,T$ taking into account the portion of the cost budget that has been expended. In this problem, we only optimize over $K_{m},\ldots,K_{T}$. This problem is again a MINLP.

Next, we look at approximate solutions to \eqref{ext:cost:rhoKnownHorizon} and \eqref{ext:cost:rhoUnknownPartial}. The major difficulties in solving these programs are that the decision variables $\{K_{n}\}_{n=1}^{T}$ are integer-valued and the cost function $p(K)$ may be discontinuous at zero due to fixed costs. We consider relaxing $K_{n}$ to be real-valued and introduce a piecewise approximation $\hat{p}(K)$ of the cost functions $p(K)$:
\[
\hat{p}(K) = \left( \frac{p(K_{0})K}{K_{0}} \right) \mathbbm{1}_{\{K \leq K_{0}\}} +  p(K)  \mathbbm{1}_{\{ K > K_{0} \}}
\]
Generally, we pick $0 < K_{0} < 1$. We consider the relaxed program
\begin{equation}
\label{ext:cost:rhoKnownRelaxation}
\arraycolsep=1.4pt\def\arraystretch{1.5}
\begin{array}{ll@{}ll}
\underset{K_{1},\ldots,K_{T}}{\text{minimize}}  & \displaystyle \phi(\xi_{1},\ldots,\xi_{T}) &\\
\text{subject to}& \displaystyle \sum_{n=1}^{T} \hat{p}(\rho,K_{n}) \leq P \\
& K_{1} \leq K_{2} & \\
& K_{n}  \leq K_{n-1} + K_{n+1} & \;\;\; n=2,\ldots,T-1 \\
& K_{T-1} \leq K_{T} & \\
& K_{n} \in \mathbb{R}_{\geq 0}  &\;\;\; n=1,\ldots,T
\end{array}
\end{equation}
We also relax the indicator constraints to inequality to encourage taking samples at consecutive times. In practice, this forces more gradual changes in samples $K_{n}$ and makes it easier to solve these problems. This problem can be readily solved by gradient based solvers such as IPOPT \cite{Wachter2006}.

When $\rho$ is unknown, we can repeatedly solve this problem using the latest estimate of $\rho$ by solving the following sequence of problems:
\begin{equation}
\label{ext:cost:rhoUnknownRelaxation}
\arraycolsep=1.4pt\def\arraystretch{1.5}
\begin{array}{ll@{}ll}
\underset{K_{n+1},\ldots,K_{T}}{\text{minimize}}  & \displaystyle \phi(\xi_{1},\ldots,\xi_{T}) &\\
\text{subject to}& \displaystyle \sum_{i=1}^{n} \hat{p}(\hat{\rho}_{i},K_{i}) \leq P - \sum_{i=n+1}^{T} \hat{p}(\hat{\rho}_{i-1},K_{i}) \\
& K_{1} \leq K_{2} & \\
& K_{n}  \leq K_{n-1} + K_{n+1} & \;\;\; n=2,\ldots,T-1 \\
& K_{T-1} \leq K_{T} & \\
& K_{n} \in \mathbb{R}_{\geq 0}  &\;\;\; n=1,\ldots,T
\end{array}
\end{equation}

%We do not have any guarantees on the performance of this approach.

\subsection{Cross Validation}
\label{ext:crossVal}

We can also apply cross-validation for model selection. Suppose we have loss functions $\ell_{\lambda}(\bw,\bz)$ parameterized by $\lambda$, which controls the model complexity. For example, we could have a quadratic penalty term
\[
\ell_{\lambda}(\bw,\bz) = \tilde{\ell}(\bw,\bz) + \frac{1}{2}\lambda \| \bw \|_{2}^{2}
\]
The value of $\lambda = 0$ corresponds to the true loss function that we want to minimize. Suppose we have $C$ different values $\lambda^{(1)},\lambda^{(2)},\ldots,\lambda^{(C)}$ of $\lambda$ under consideration. For each $\lambda^{(i)}$, we generate an approximate minimizer $\bw_{n}^{(i)}$ of
\begin{equation}
\label{ext:crossval:lambdaI}
\mathbb{E}_{\bz_{n} \sim p_{n}}\left[  \ell_{\lambda^{(i)}}(\bw,\bz_{n}) \right]
\end{equation}
We want to select the value $\lambda^{(i)}$ and corresponding $\bw_{n}^{(i)}$ that achieves the smallest loss
\begin{equation}
\label{ext:crossval:lambda0}
\mathbb{E}_{\bz_{n} \sim p_{n}} \left[ \ell_{0}(\bw_{n}^{(i)},\bz_{n}) \right]
\end{equation}
We generate an approximate minimizer $\bw_{n}^{(i)}$ for each problem in \eqref{ext:crossval:lambdaI} starting from $\bw_{n-1}^{(i)}$. To select the best choice of $\lambda^{(i^{*})}$ in terms of minimizing \eqref{ext:crossval:lambda0}, we apply cross-validation and set $\bw_{n} = \bw_{n}^{(i^{*})}$ \cite{Hastie2001}.

The idea behind cross-validation is to divide the training samples $\{\bz_{n}(k)]\}_{k=1}^{K_{n}}$ into $P$ equal sized pieces. For every $P-1$ out of $P$ pieces, we use the $P-1$ pieces of the training set to generate an approximate solution $\tilde{\bw}_{n}^{(i)}$ to \eqref{ext:crossval:lambdaI}. We use the remaining piece of the training set to evaluate the empirical test loss achieved by $\tilde{\bw}_{n}^{(i)}$ using a sample average approximation. We do this for every possible choice of $P-1$ out of $P$ pieces and average the empirical test loss estimates. We then select the value $\lambda^{(i^{*})}$ that achieves the smallest empirical test loss.

To apply cross-validation to our framework, we run $C$ parallel versions of our approach and at time $n$ we generate $C$ different choices for the number of samples $K_{n}^{(i)}$. We then choose
\[
K_{n} = \max\{K_{n}^{(1)},\ldots,K_{n}^{(C)}\}
\]
After choosing $K_{n}$, we apply the usual cross-validation approach to select $\lambda^{(i)}$ for time $n$.
Fig.~\ref{withRhoUnknown:crossValApproach} shows this approach for two values of $\lambda$.

\begin{figure}[!ht]
	\centering
	\includegraphics[width = 1\linewidth]{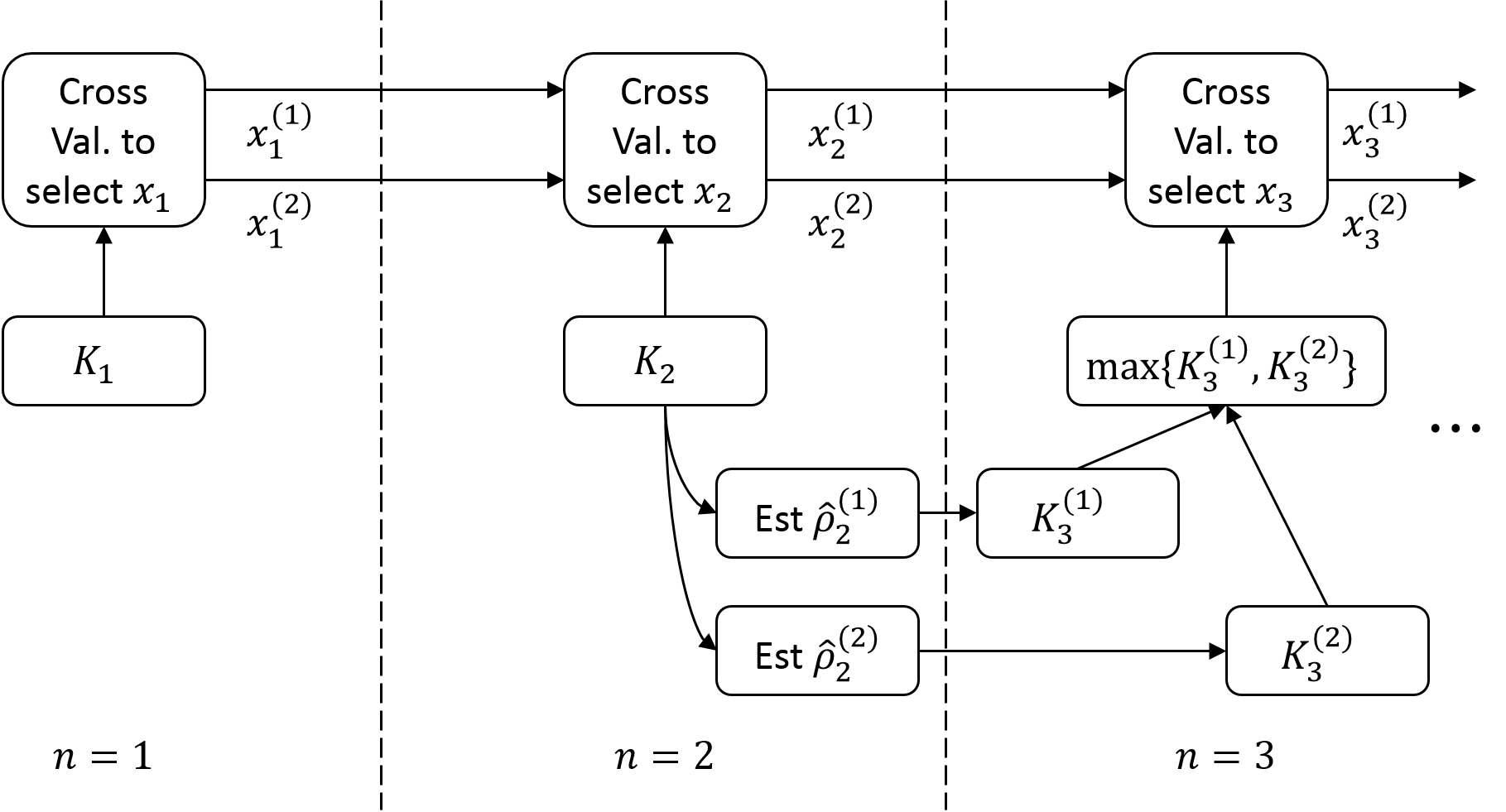}
	\caption{Cross validation approach}
	\label{withRhoUnknown:crossValApproach}
\end{figure}	

\section{Experiments}
\label{exper}

We provide two regression examples for synthetic and real data as well as a classification example for synthetic data. For the synthetic regression problem, we can explicitly compute $\rho$ and $\bw_{n}^{*}$ and exactly evaluate the performance of our method. It is straightforward to check that all requirements in \ref{probState:assump1}-\ref{probState:assump6} are satisfied for the problems considered in this section. We apply the ``do not update past excess risk" choice of $K_{n}$ here.

\subsection{Synthetic Regression}
\label{c:mlApplications:synReg}

Consider a regression problem with synthetic data using the penalized quadratic loss
\[
\ell(\bw,\bz) = \frac{1}{2} \left( y - \bx^{\top} \bw \right)^{2} + \frac{1}{2} \lambda \| \bw \|_{2}^{2}
\]
with $\bz = (\bx,y) \in \mathbb{R}^{3}$. We further assume that
\[
\bz_{n} \sim \mathcal{N}\left( \bm{0}, \left[ \begin{array}{cc}
\sigma_{\bx}^{2} \bm{I} & r_{\bx_{n},y_{n}} \\
r_{\bx_{n},y_{n}}^{\top} & \sigma_{y_{n}}^{2}
\end{array} \right] \right)
\]
Under these assumptions, we can analytically compute minimizers $\bw_{n}^{*}$ of $f_{n}(\bw) = \mathbb{E}_{\bz_{n} \sim p_{n}} \left[ \ell(\bw,\bz_{n})  \right]$. We change only $r_{\bx_{n},y_{n}}$ and $\sigma_{y_{n}}^{2}$ appropriately to ensure that $\| \bw_{n}^{*} - \bw_{n-1}^{*} \|_{2} = \rho$ holds for all $n$. We find approximate minimizers using SGD with $\lambda = 0$. We estimate $\rho$ using the direct estimate.

We let $n$ range from $1$ to $20$ with $\rho = 1$, a target excess risk $\epsilon = 0.1$, and $K_{n}$ from \eqref{withRhoUnknown:KnChoice}. We average over twenty runs of our algorithm. \figurename{}~\ref{exper:synthRegress:rhoEst} shows $\hat{\rho}_{n}$, our estimate of $\rho$, which is above $\rho$ in general. \figurename{}~\ref{exper:synthRegress:Kn} shows the number of samples $K_{n}$, which settles down. We can exactly compute $f_{n}(\bw_{n}) - f_{n}(\bw_{n}^{*})$, and so by averaging over the twenty runs of our algorithm, we can estimate the excess risk (denoted ``sample average estimate''). We over the time horizon from $n=1$ to $25$ to yield the sample average estimate excess risk given by $2.797 \times 10^{-2} \pm 1.071 \times 10^{-2}$. Therefore, we see that we achieve our desired excess risk.

\iftoggle{useTwoColumn}{
	\begin{figure}[!ht]
		\centering
		\includegraphics[width = \linewidth]{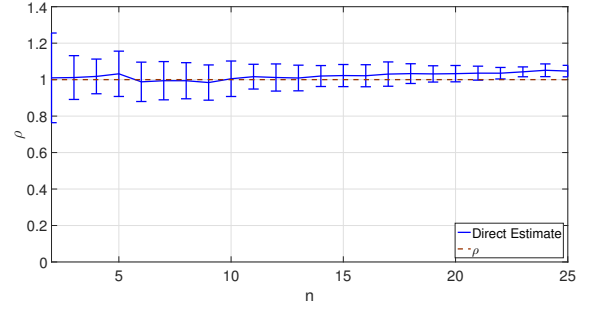}
		\caption{$\rho$ estimate for synthetic regression.}
		\label{exper:synthRegress:rhoEst}
	\end{figure}
	\begin{figure}[!ht]
		\centering
		\includegraphics[width = \linewidth]{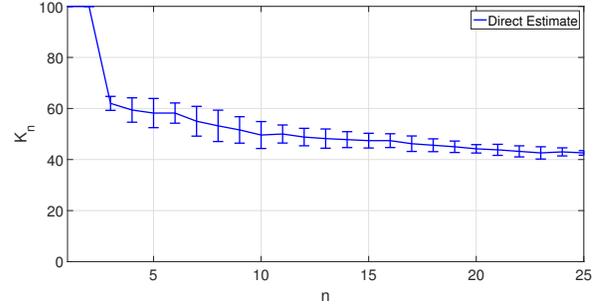}
		\caption{$K_{n}$ for synthetic regression.}
		\label{exper:synthRegress:Kn}
	\end{figure}
}{
\begin{figure}[!ht]
	\centering
	\begin{minipage}[b]{0.48\linewidth} \quad
		\centering
		\includegraphics[width = \linewidth]{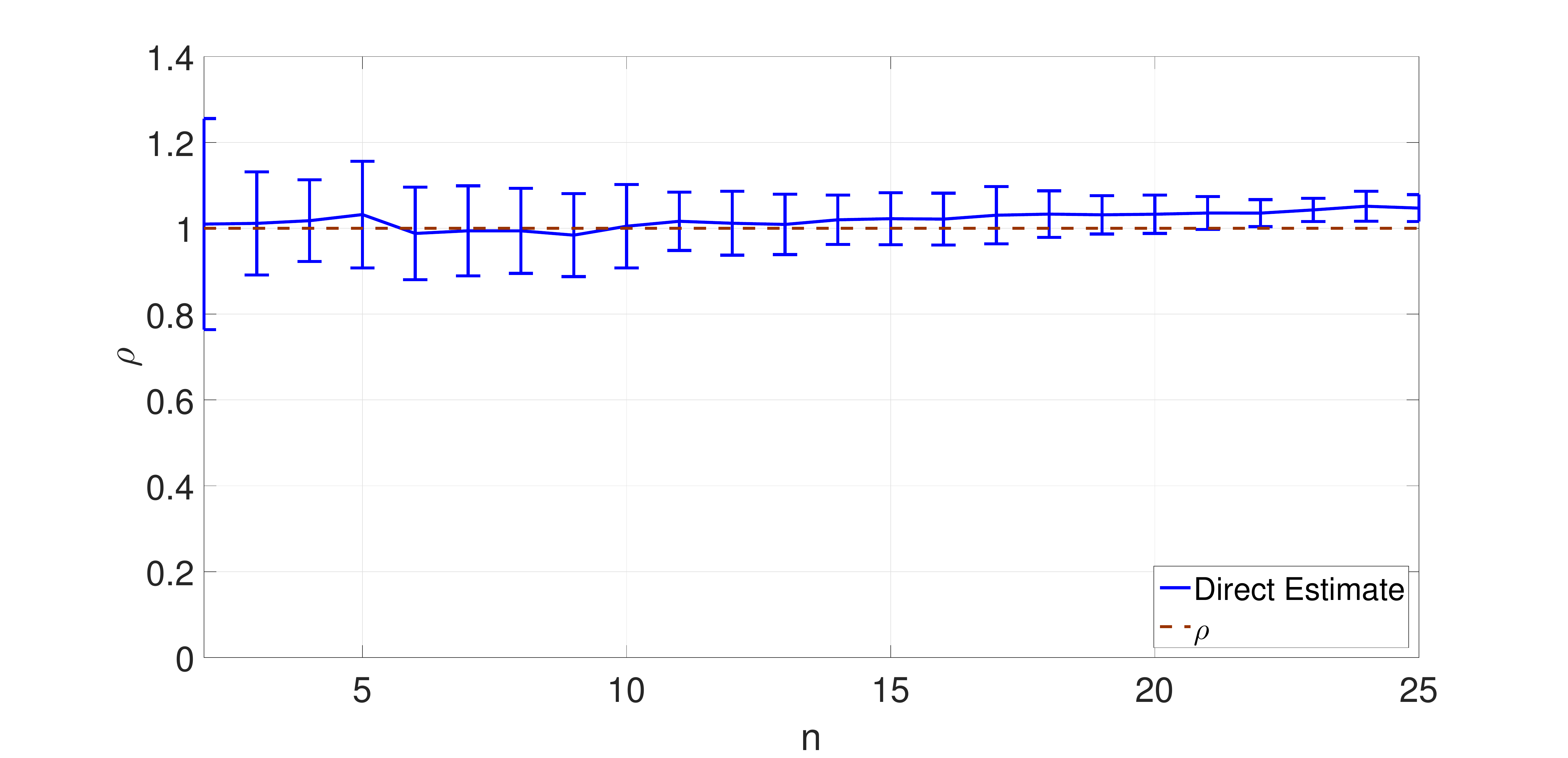}
		\caption{$\rho$ Estimate}
		\label{exper:synthRegress:rhoEst}
	\end{minipage}
	\begin{minipage}[b]{0.48\linewidth}
		\centering
		\includegraphics[width = \linewidth]{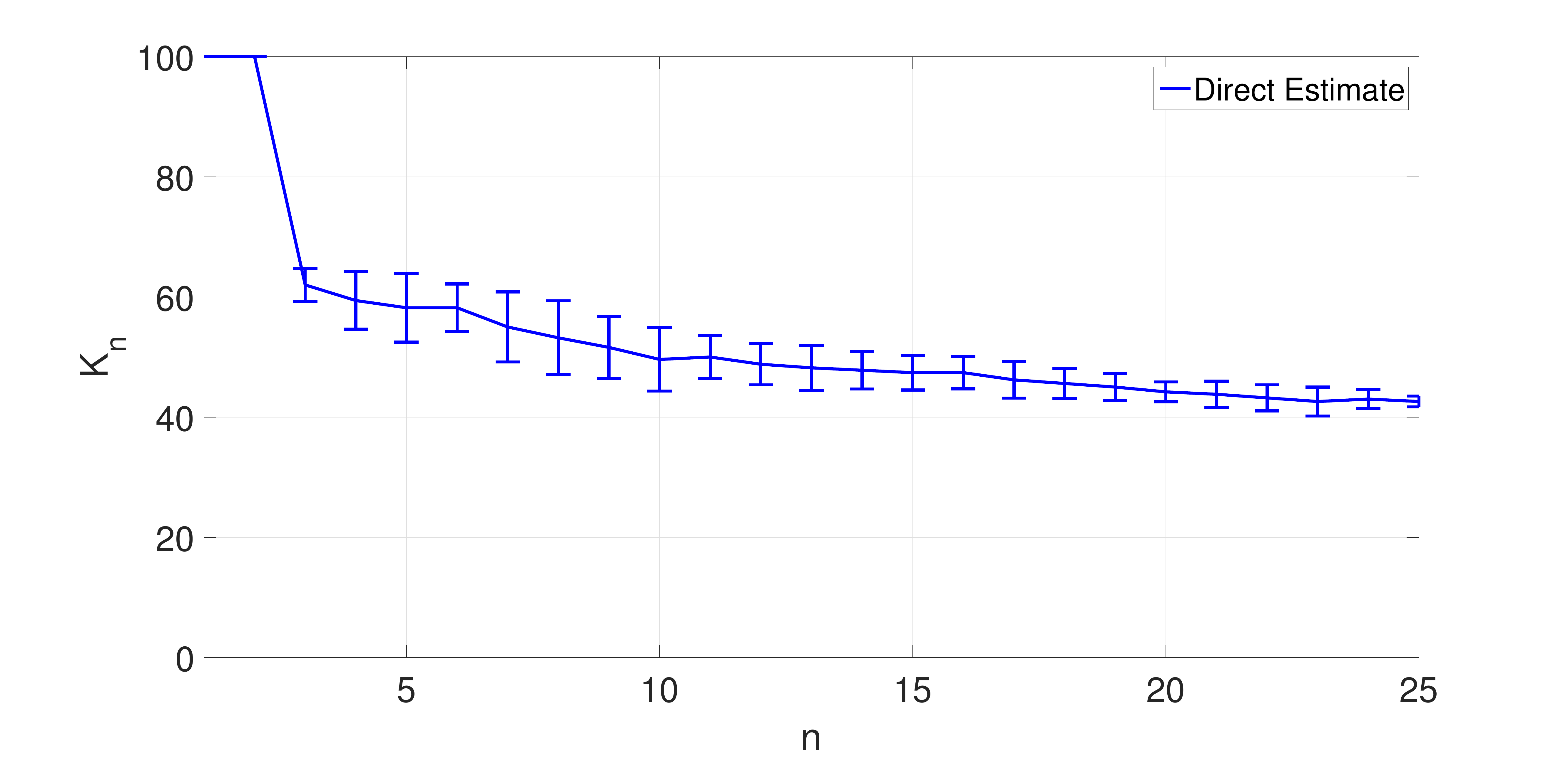}
		\caption{$K_{n}$}
		\label{exper:synthRegress:Kn}
	\end{minipage}
\end{figure}
}

\subsubsection{Cost Approach}
We consider applying the cost approach in Section~\ref{ext:cost} to the synthetic regression problem with the cost in \eqref{ext:cost:costExample}. We compare the optimal cost approach introduced in \eqref{ext:cost:rhoKnownRelaxation} of Section~\ref{ext:cost} to the approach in \eqref{withRhoUnknown:KnChoice}, taking all samples at time $n=1$ as in \eqref{ext:cost:costUpFront}, and taking samples every five time instants as in \eqref{ext:cost:costPeriodic}. Note that the method from \eqref{withRhoUnknown:KnChoice} does not satisfy the cost budget. Fig.~\ref{exper:synthRegress:cost:testLoss} shows the test loss of these approaches. We achieve similar test loss to the method in \eqref{withRhoUnknown:KnChoice} and better than the other two methods.  Fig.~\ref{exper:synthRegress:cost:Kn} shows the number of samples selected for both methods. At some time instants, our optimal cost approach does not take samples.

\iftoggle{useTwoColumn}{
	\begin{figure}[!ht]
		\centering
		\iftoggle{useMeanGap}{
			\includegraphics[width = \linewidth]{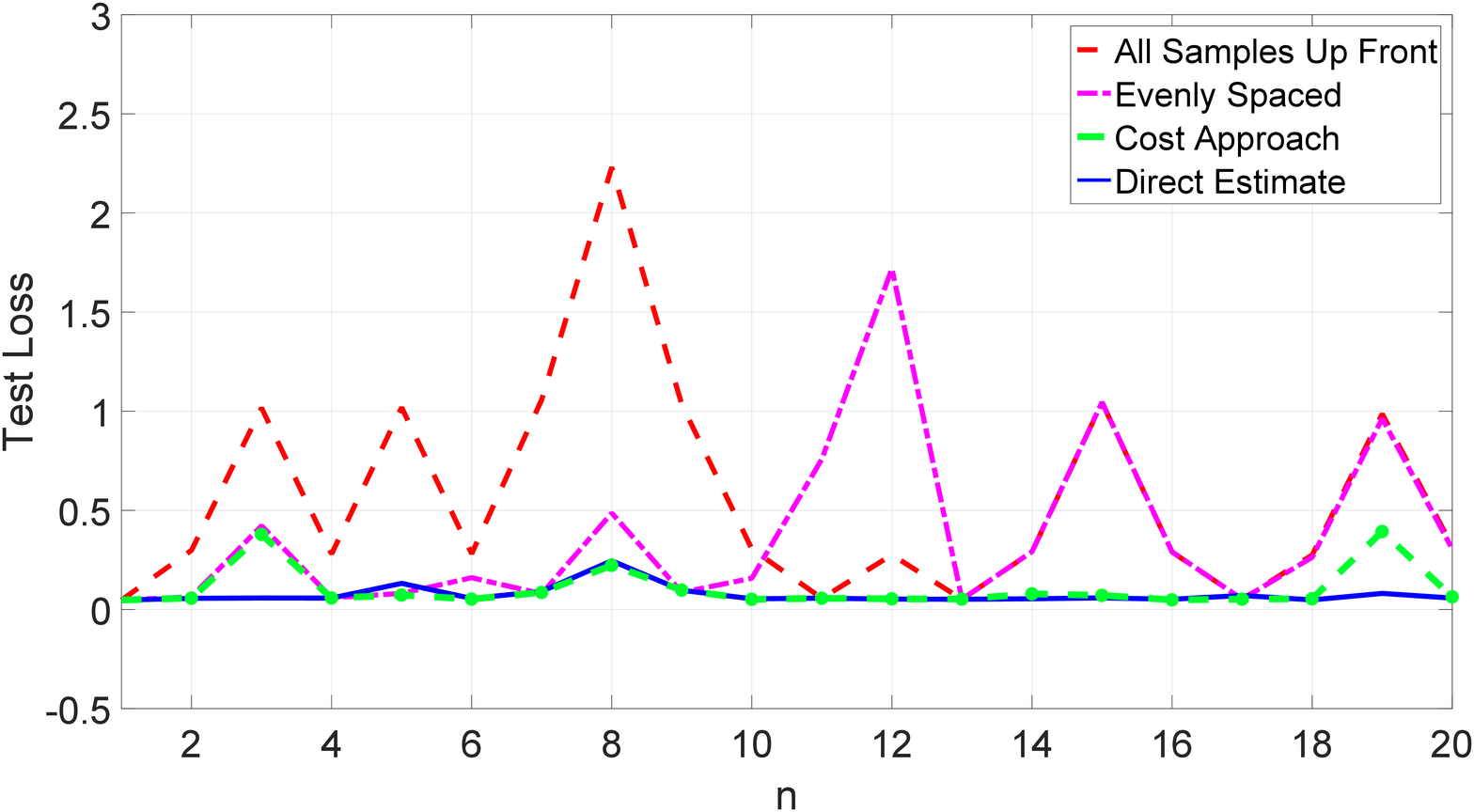}
		}{
		\includegraphics[width = \linewidth]{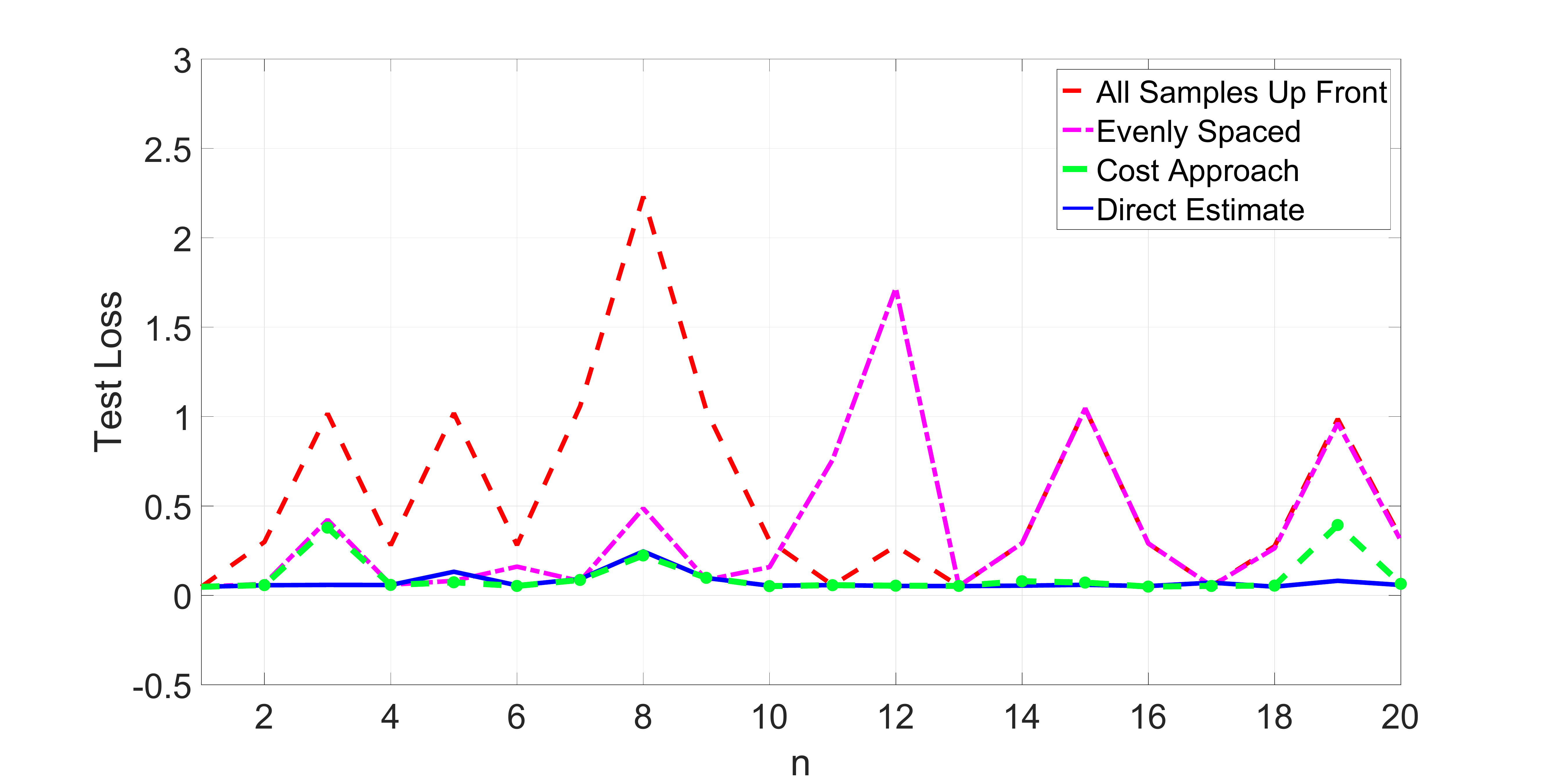}
	}
	\caption{Test Loss for synthetic regression with cost approach.}
	\label{exper:synthRegress:cost:testLoss}
\end{figure}
\begin{figure}[!ht]
	\centering
	\iftoggle{useMeanGap}{
		\includegraphics[width = \linewidth]{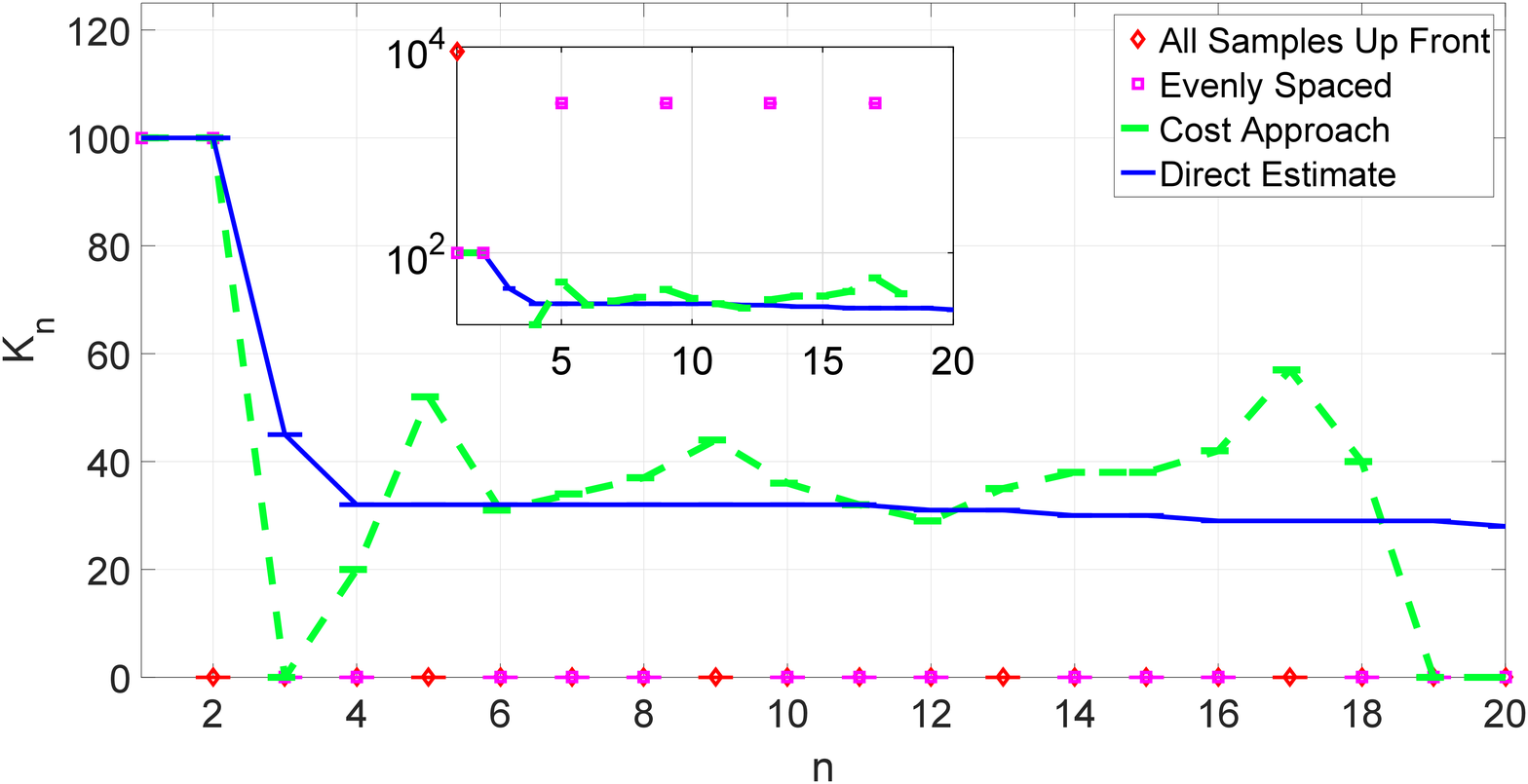}
	}{
	\includegraphics[width = \linewidth]{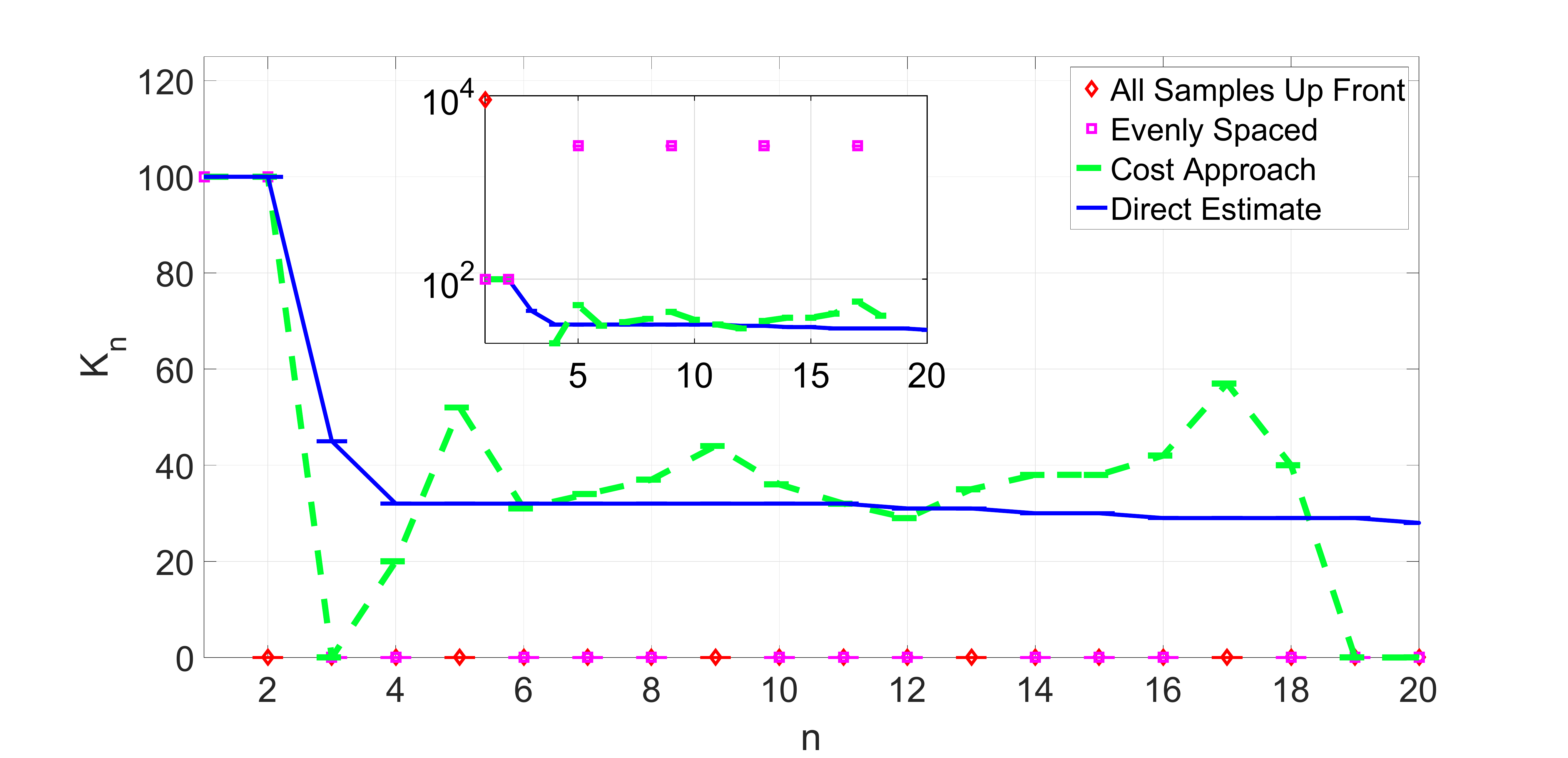}
}
\caption{Cost Choice of $K_{n}$ for synthetic regression.}
\label{exper:synthRegress:cost:Kn}
\end{figure}
}{
\begin{figure}[!ht]
	\centering
	\begin{minipage}[b]{0.48\linewidth} \quad
		\centering
		\iftoggle{useMeanGap}{
			\includegraphics[width = \linewidth]{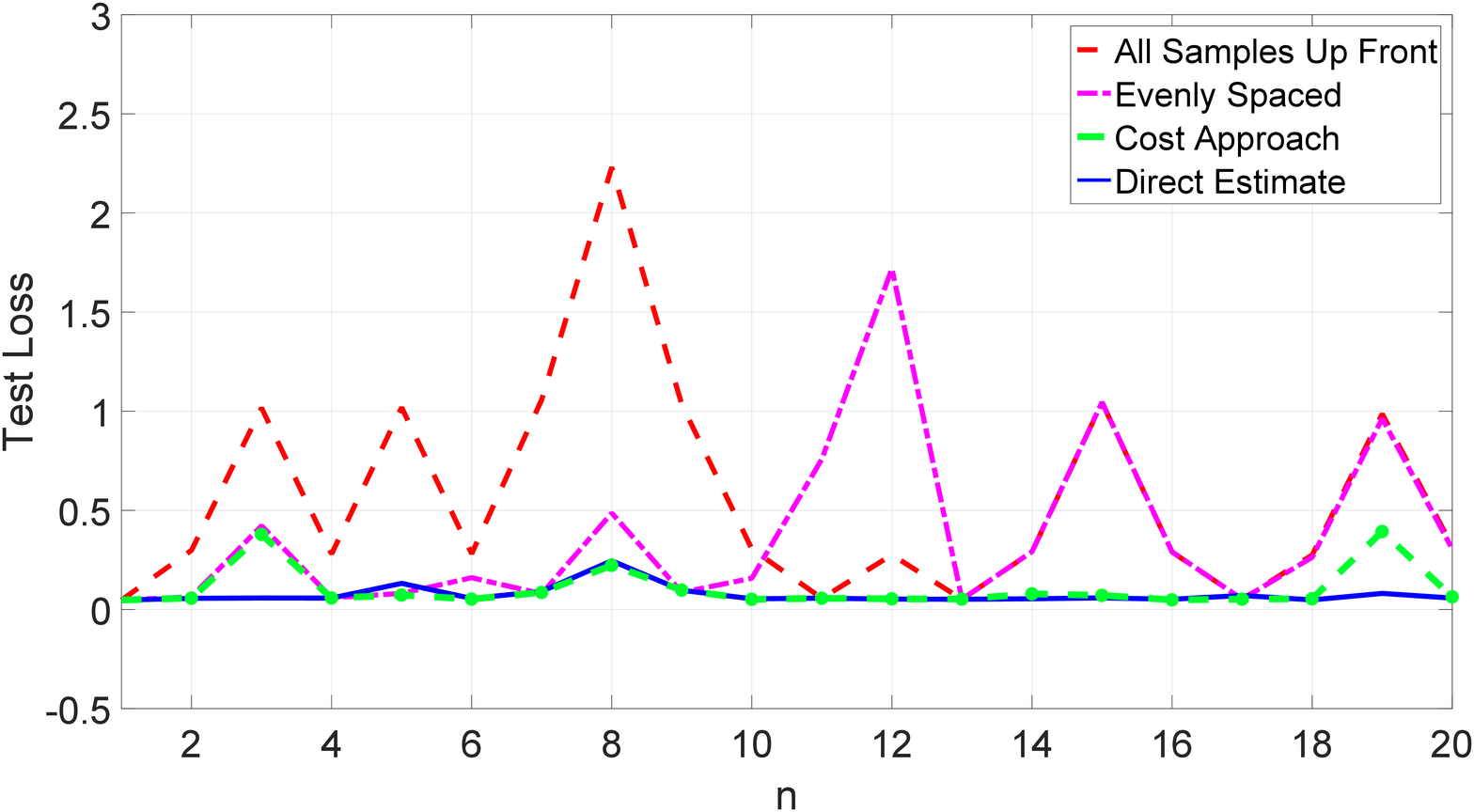}
		}{
		\includegraphics[width = \linewidth]{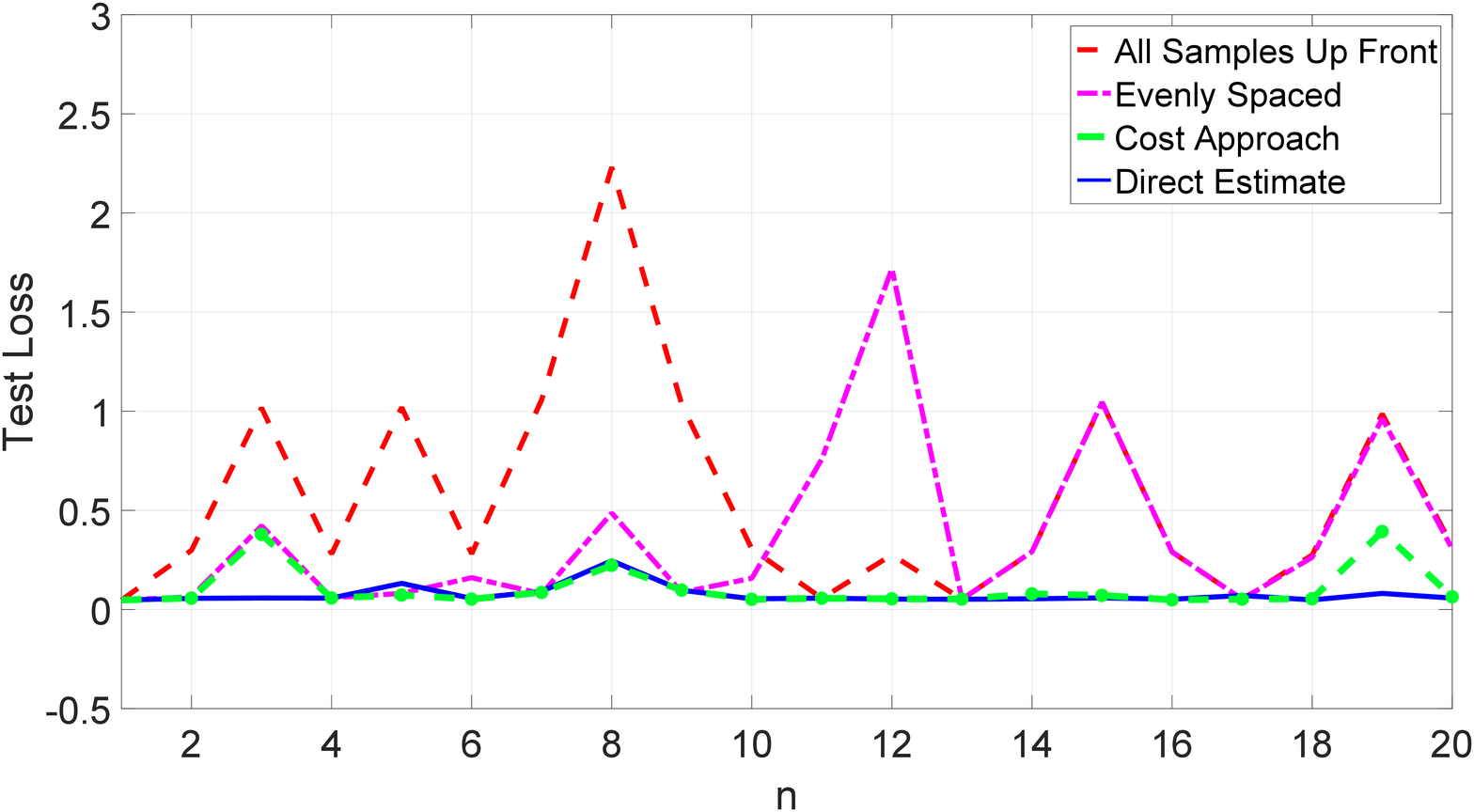}
	}
	\caption{Test Loss for synthetic regression with cost approach.}
	\label{exper:synthRegress:cost:testLoss}
\end{minipage}
\begin{minipage}[b]{0.48\linewidth}
	\centering
	\iftoggle{useMeanGap}{
		\includegraphics[width = \linewidth]{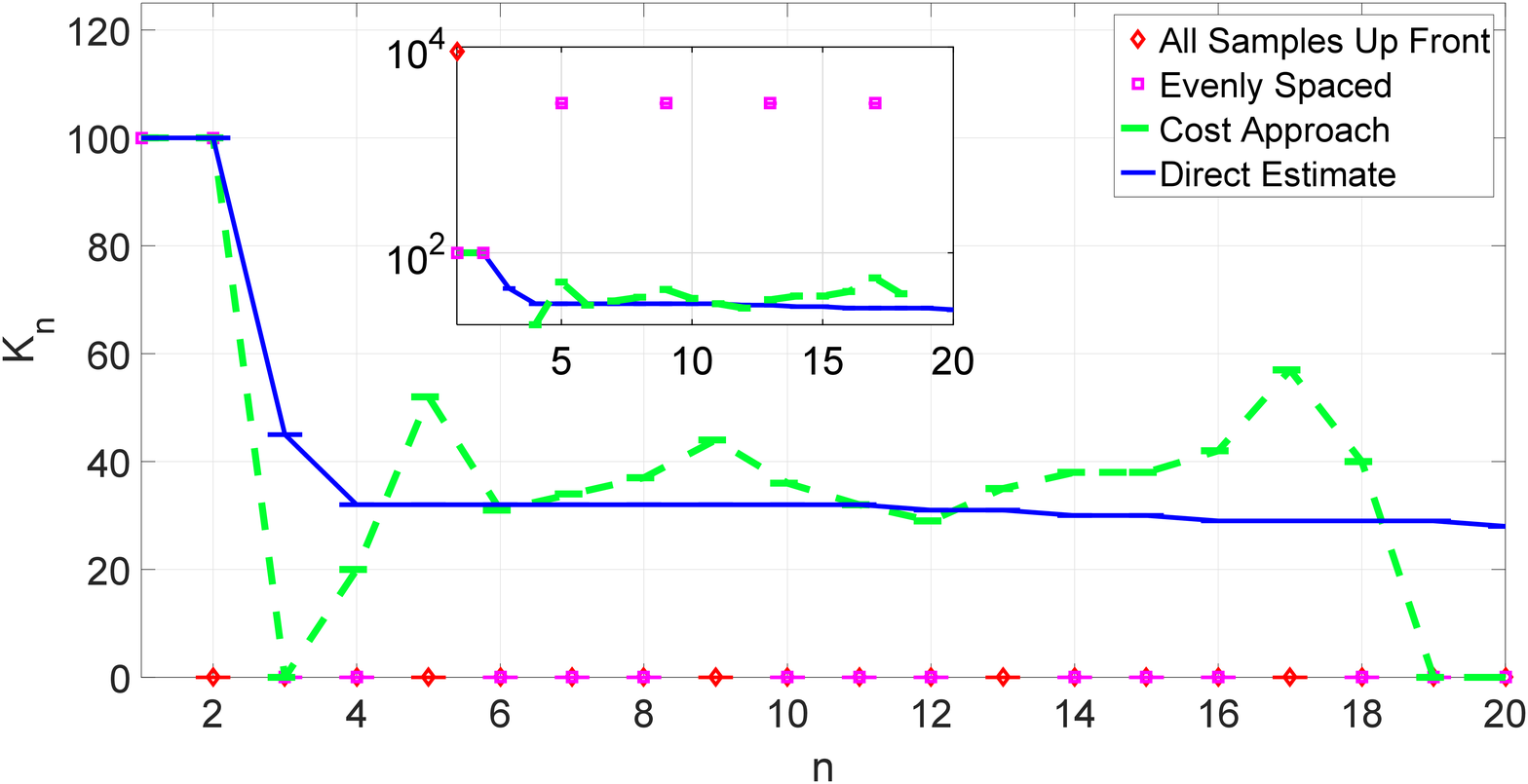}
	}{
	\includegraphics[width = \linewidth]{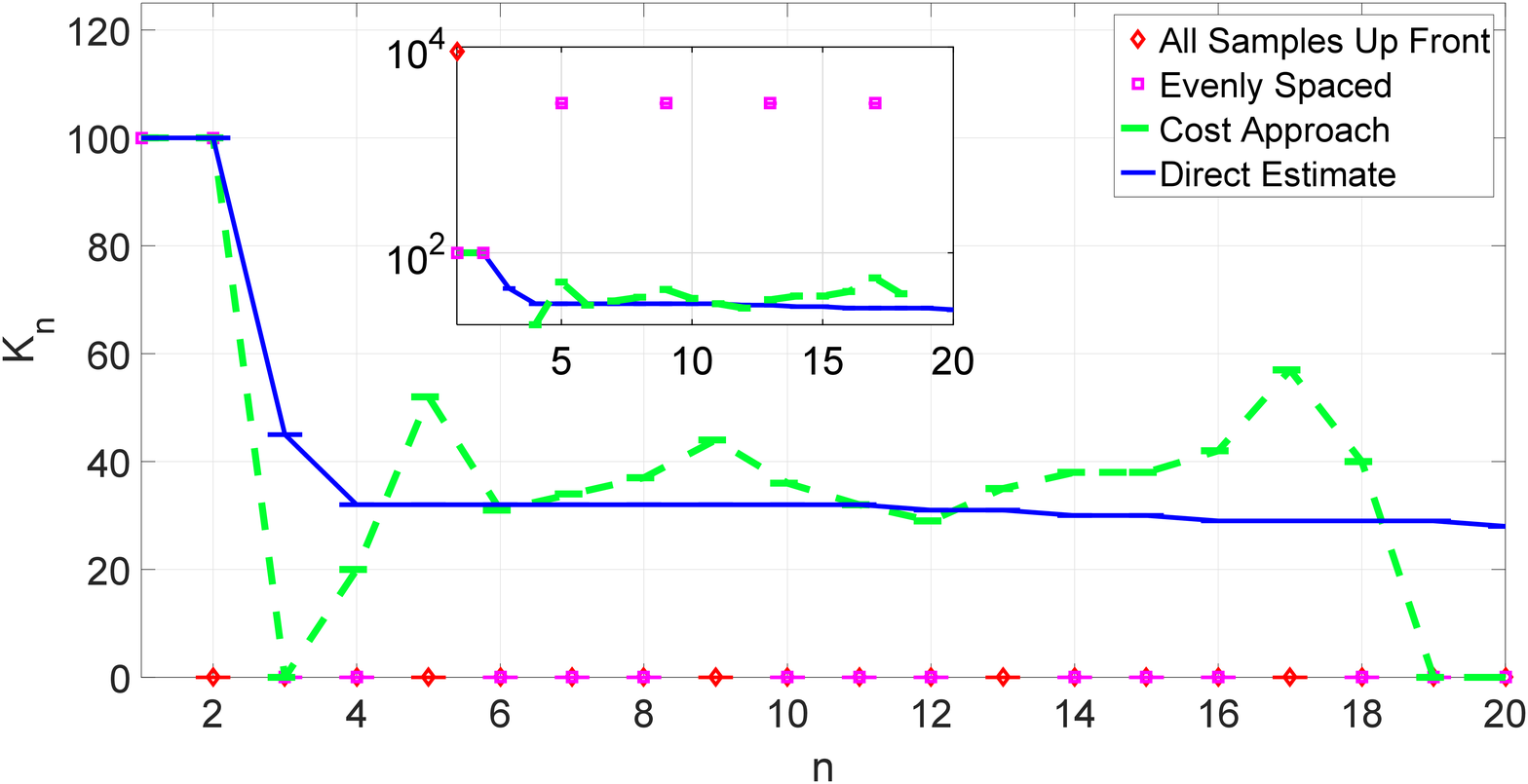}
}
\caption{Cost Choice of $K_{n}$ for synthetic regression.}
\label{exper:synthRegress:cost:Kn}
\end{minipage}
\end{figure}
}

This problem is an example of one where the initial distance term in $b(d_{0},K)$ per the discussion from Section~\ref{rhoBoundOvershoot} matters. This is evidenced by the fact that when we do not take samples after the first time instant the test loss can grow large quickly as shown in Fig.~\ref{exper:synthRegress:cost:testLoss}.

\subsection{Synthetic Classification}

Consider a binary classification problem using
\[
\ell(\bw,\bz) = \frac{1}{2} ( 1 - y ( \bx^{\top} \bw ) )_{+}^{2} + \frac{1}{2} \lambda \| \bw \|_{2}^{2}
\]
with ${\bz = (\bx,y) \in \mathbb{R}^{d} \times \mathbb{R}}$ and $(y)_{+} = \max\{y,0\}$. This is a smoothed version of the hinge loss used in support vector machines (SVM) \cite{Hastie2001}. We suppose that at time $n$, the two classes have features drawn from a Gaussian distribution with covariance matrix $\sigma^{2} \bm{I}$ but different means $\mu_{n}^{(1)}$ and $\mu_{n}^{(2)}$, i.e.,
\[
\bx_{n} \;|\; \{y_{n} = i\} \;\sim\; \mathcal{N}(\mu_{n}^{(i)} , \sigma^{2} \bm{I})
\]
The class means move slowly over uniformly spaced points on a unit sphere in $\mathbb{R}^{d}$ as in \figurename{}~\ref{exper:synthClass:ClassMeans} to ensure that the constant Euclidean norm condition defined in \eqref{probState:slowChangeConstDef} holds. We find approximate minimizers using SGD with $\lambda = 0.1$. We estimate $\rho$ using the direct estimate.

\begin{figure}[!ht]
	\centering
	\includegraphics[width = 0.4\linewidth]{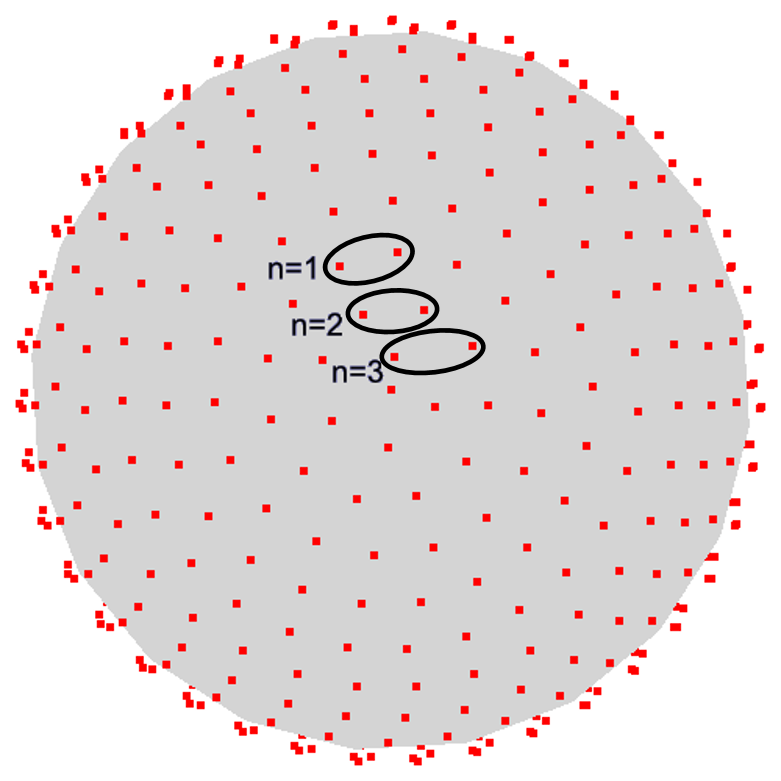}
	\caption{Evolution of Class Means}
	\label{exper:synthClass:ClassMeans}
\end{figure}

We let $n$ range from $1$ to $25$ and target a excess risk $\epsilon = 0.1$. We average over twenty runs of our algorithm. As a comparison, if our algorithm takes $\{K_{n}\}_{n=1}^{25}$ samples, then we consider taking $\sum_{n=1}^{25} K_{n}$ samples up front at $n=1$. This is what we would do if we assumed that our problem is not time varying. \figurename{}~\ref{exper:synthClass:rhoEst} shows $\hat{\rho}_{n}$, our estimate of $\rho$. \figurename{}~\ref{exper:synthClass:testLoss} shows the average test loss for both sampling strategies. To compute test loss we draw $T_{n}$ additional samples $\{\bz_{n}^{\text{test}}(k)\}_{k=1}^{T_{n}}$ from $p_{n}$ and compute $\frac{1}{T_{n}} \sum_{k=1}^{T_{n}} \ell(\bw_{n},\bz_{n}^{\text{test}}(k))$. We see that our approach achieves substantially smaller test loss than taking all samples up front. We do not draw the error bars on this plot as it makes it difficult to see the actual losses achieved.

To further evaluate our approach we look at the receiver operating characteristic (ROC) of our classifiers. The ROC is a plot of the probability of a true positive against the probability of a false positive. The area under the curve (AUC) of the ROC equals the probability that a randomly chosen positive instance ($y = 1$) will be rated higher than a negative instance ($y = -1$) \cite{Fawcett2006}. Thus, a large AUC is desirable. \figurename{}~\ref{exper:synthClass:AUC} plots the AUC of our approach against taking all samples up front. Our sampling approach achieve a substantially larger AUC.

\iftoggle{useTwoColumn}{
	\begin{figure}[!ht]
		\centering
		\includegraphics[width = \linewidth]{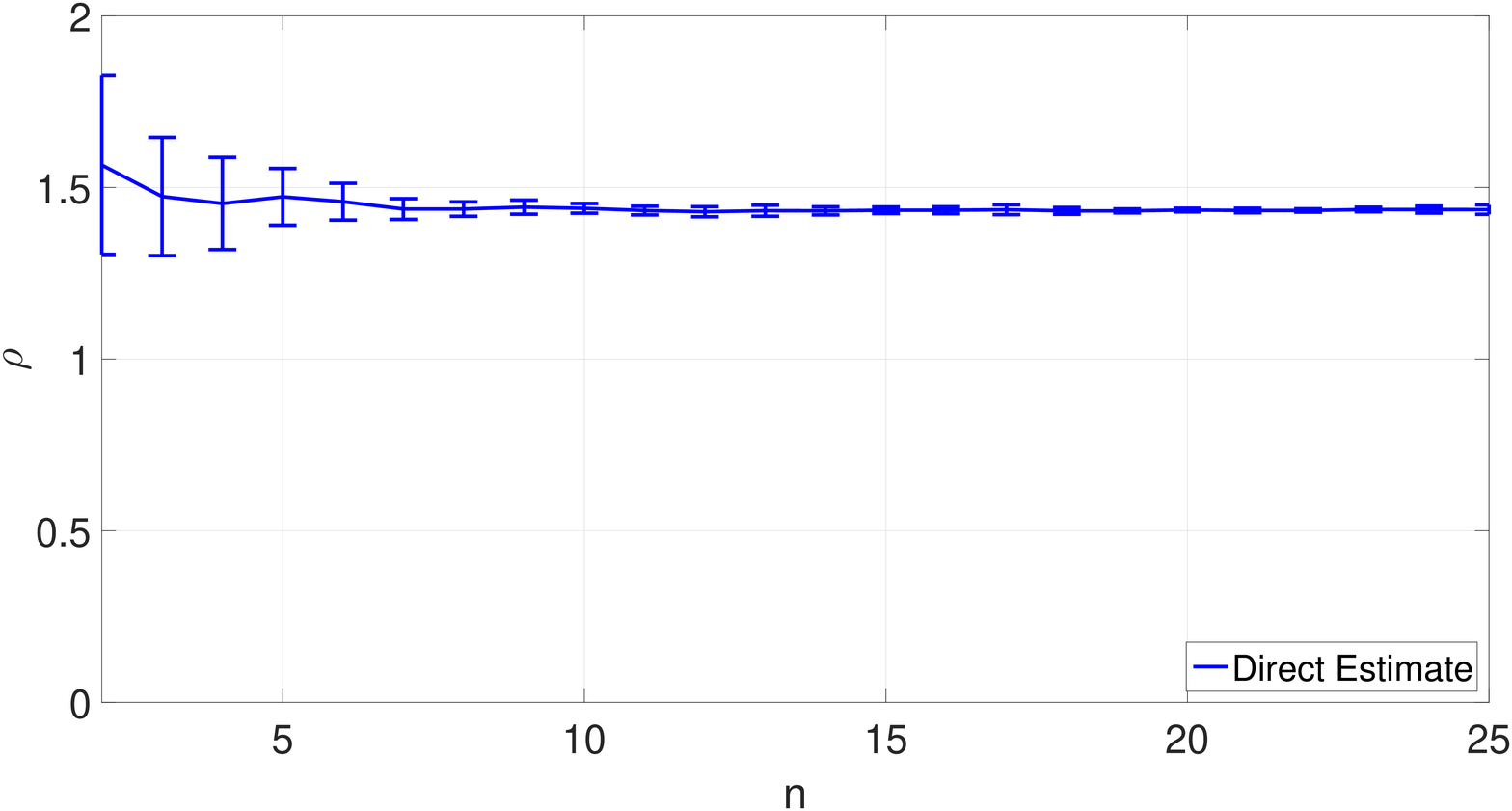}
		\caption{$\rho$ estimate for synthetic classification.}
		\label{exper:synthClass:rhoEst}
	\end{figure}%
	\begin{figure}[!ht]
		\centering
		\includegraphics[width = \linewidth]{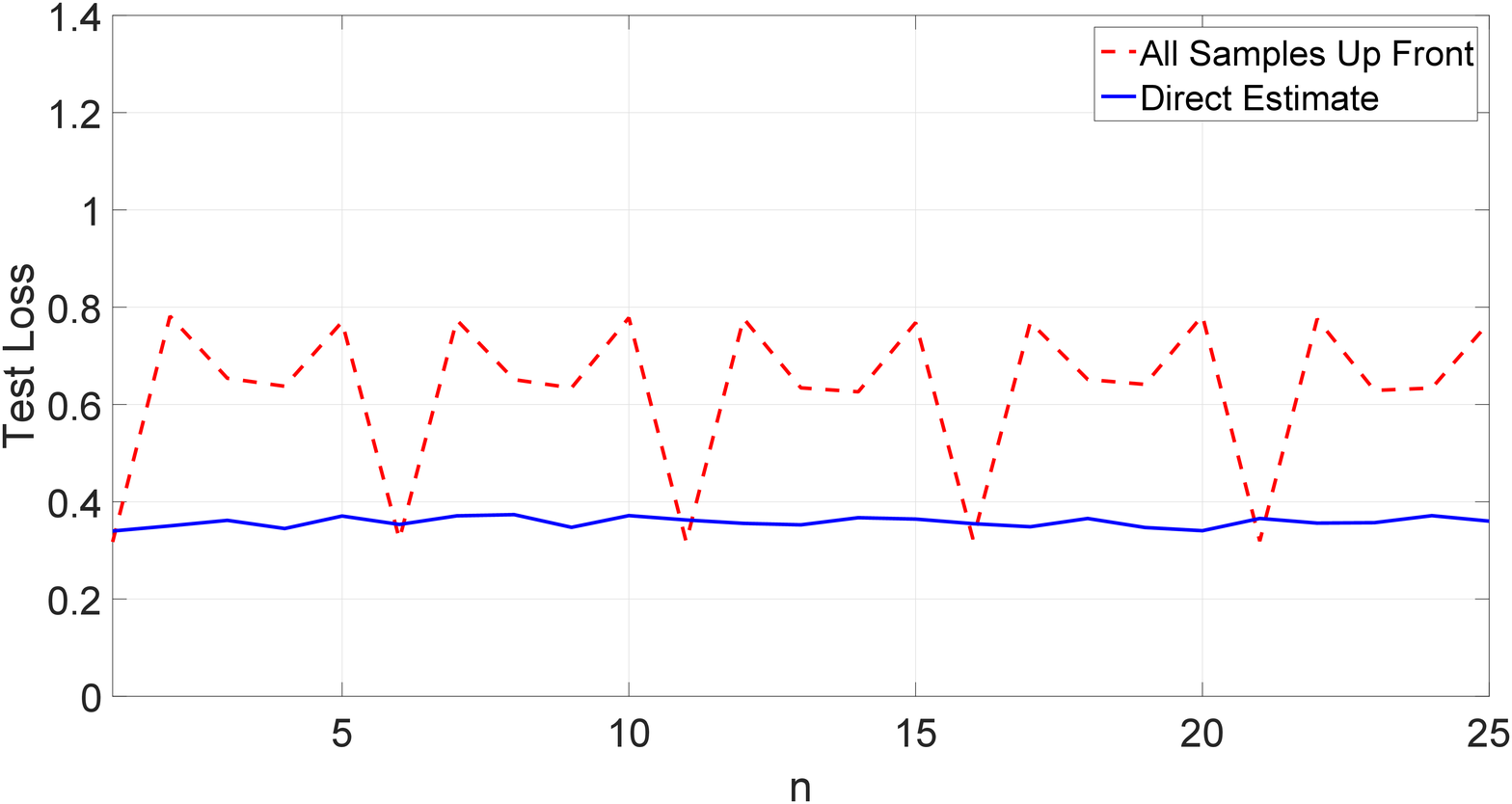}
		\caption{Test loss for synthetic classification.}
		\label{exper:synthClass:testLoss}
	\end{figure}%	
	\begin{figure}[!ht]
		\centering
		\includegraphics[width = \linewidth]{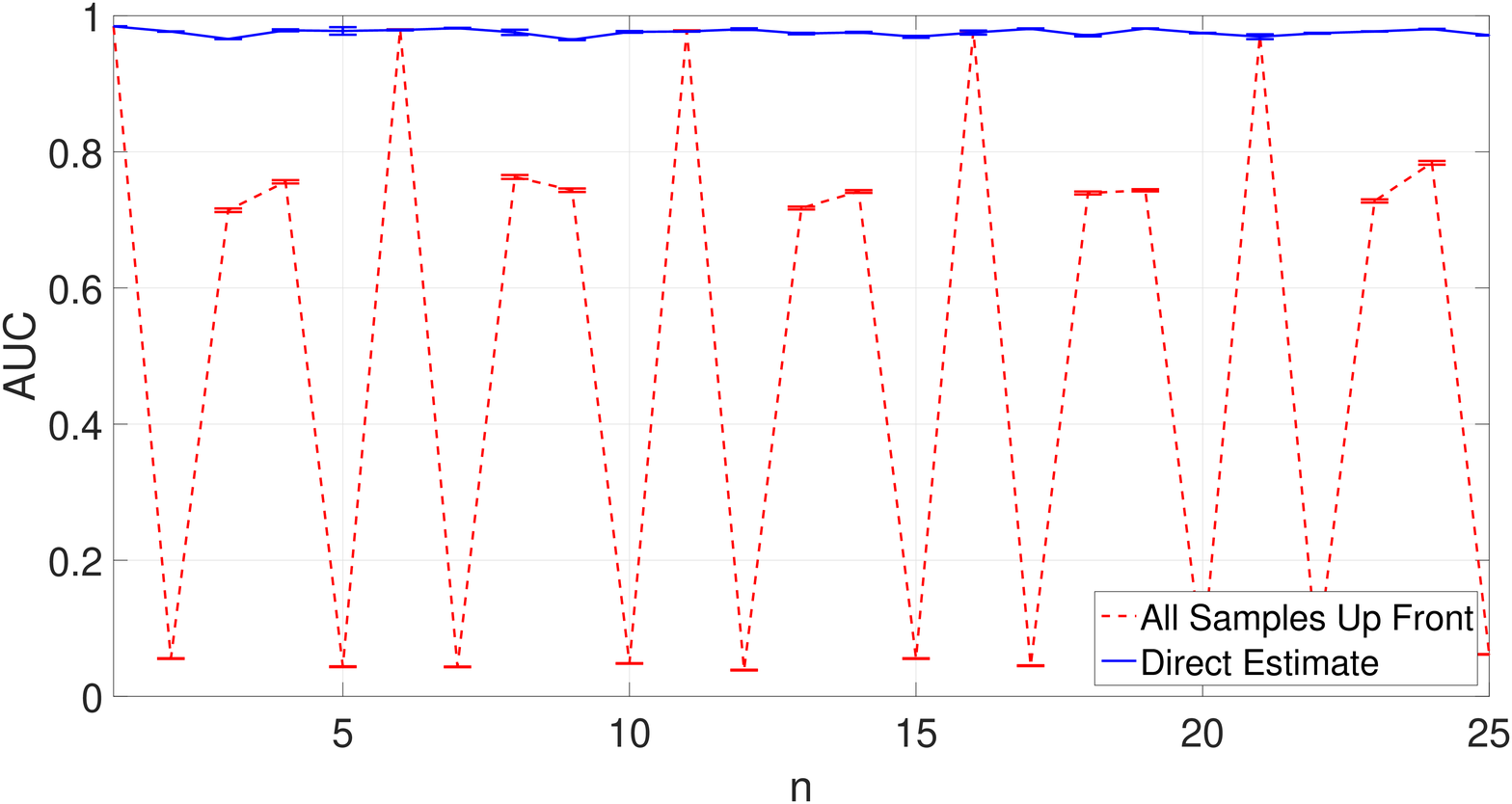}
		\caption{Area Under the Curve for synthetic classification.}
		\label{exper:synthClass:AUC}
	\end{figure}%	
}{
\begin{figure}[!ht]
	\centering
	\begin{minipage}{0.5\textwidth}
		\centering
		\includegraphics[width = 1 \textwidth]{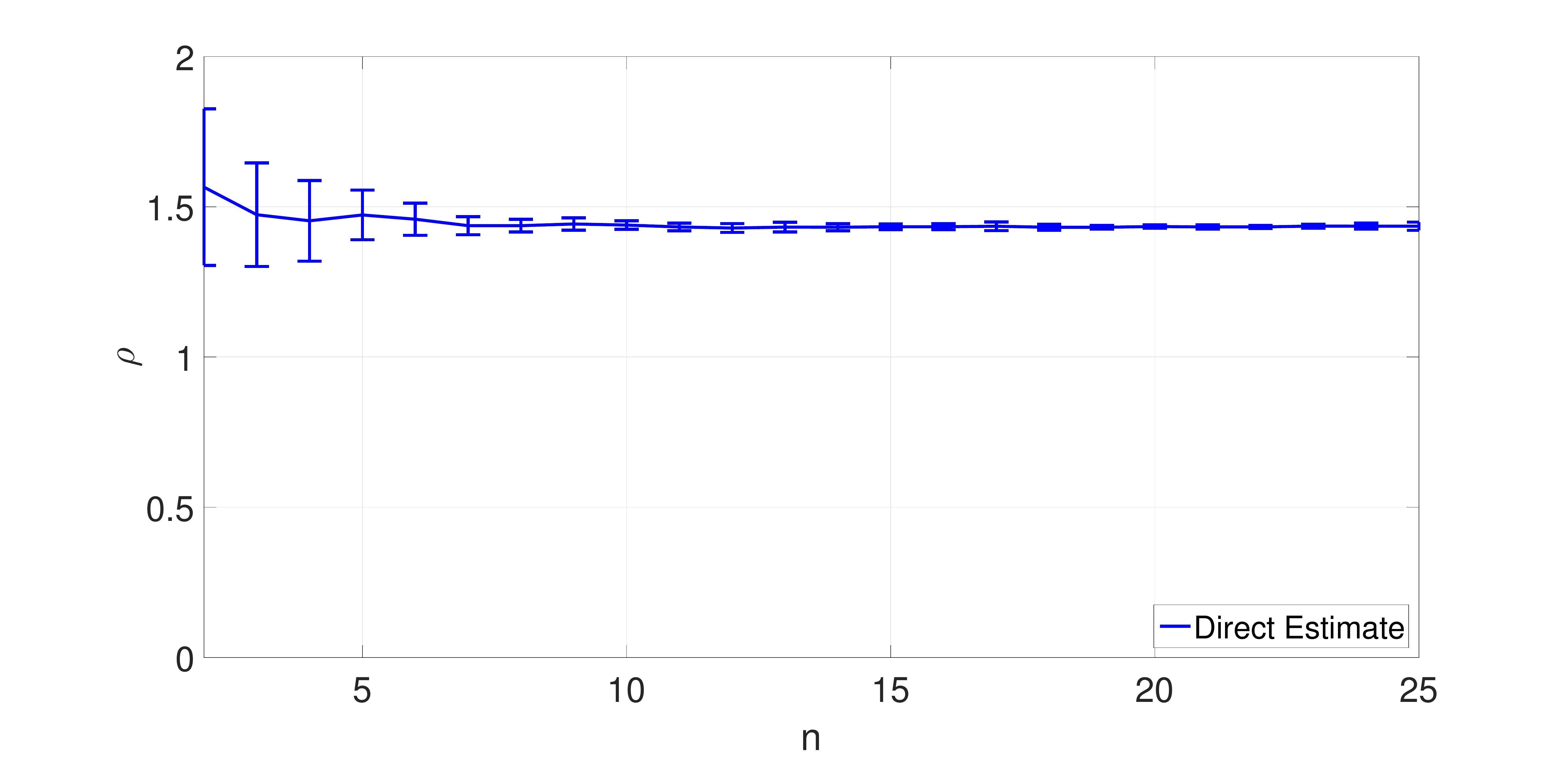}
		\caption{$\rho$ estimate for synthetic classification.}
		\label{exper:synthClass:rhoEst}
	\end{minipage}%
	\begin{minipage}{0.5\textwidth}
		\centering
		\includegraphics[width = 1 \textwidth]{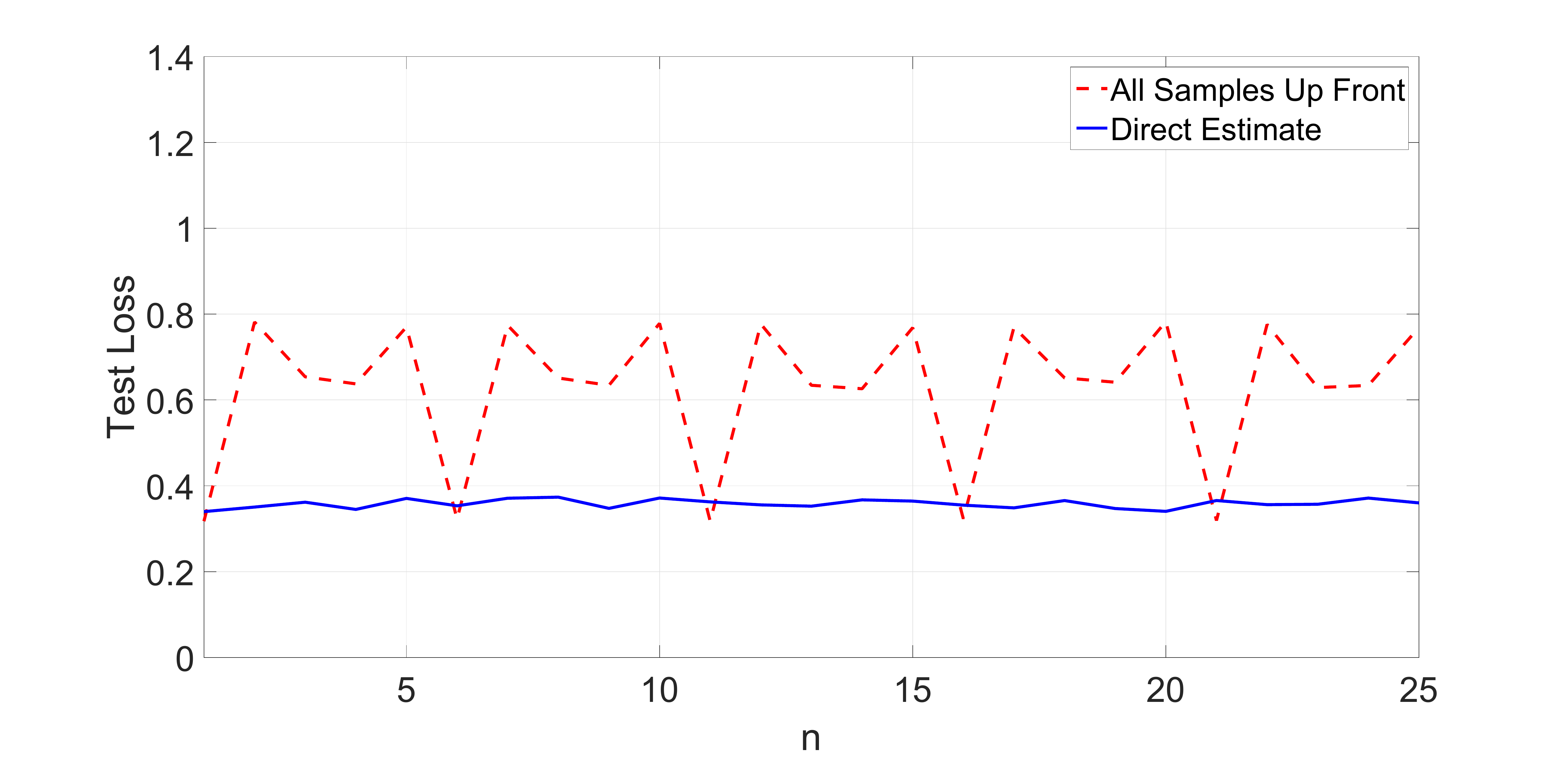}
		\caption{Test loss for synthetic classification.}
		\label{exper:synthClass:testLoss}
	\end{minipage}%
\end{figure}
}

\subsection{Panel Study on Income Dynamics Income - Regression}
The Panel Study of Income Dynamics (PSID) surveyed individuals every year to gather demographic and income data annually from 1974-2012 \cite{PSID2015}. We want to predict an individual's annual income ($y$) from several demographic features ($\bx$) including age, education, work experience, etc. chosen based on previous economic studies in \cite{Murphy1990}. The idea of this problem conceptually is to rerun the survey process and determine how many samples we would need if we wanted to solve this regression problem to within a desired excess risk criterion $\epsilon$.

We use the same loss function, direct estimate for $\rho$, and minimization algorithm as the synthetic regression problem. We average over twenty runs of our algorithm by resampling without replacement \cite{Hastie2001}. For the sake of comparison, given a choice of samples $\{K_{n}\}_{n=1}^{T}$ produced by our approach, we compare against taking $\sum_{n=1}^{T} K_{n}$ samples at time $n=1$ and none afterwards. Note that this is what we would do if we believed that the regression model does not change over time. We are aware of no other methods to select the number of samples $K_{n}$ to control the excess risk against which we could compare our approach.

\figurename{}~\ref{exper:psid:Kn} shows the number of samples $K_{n}$, which settles down quickly. \figurename{}~\ref{exper:psid:rhoEst} shows $\hat{\rho}_{n}$. \figurename{}~\ref{exper:psid:testLosses} shows the test losses over time evaluated over twenty percent of the available samples. The test loss for our approach is substantially less than that obtained by taking the same number of samples up front. %The square root of the average test loss over this time period for our approach and all samples up front are $\$1153 \pm 352$ and $\$2805 \pm 424$ respectively in 1997 dollars.

\iftoggle{useTwoColumn} {
	\begin{figure}[!ht]
		\centering
		\includegraphics[width = \linewidth]{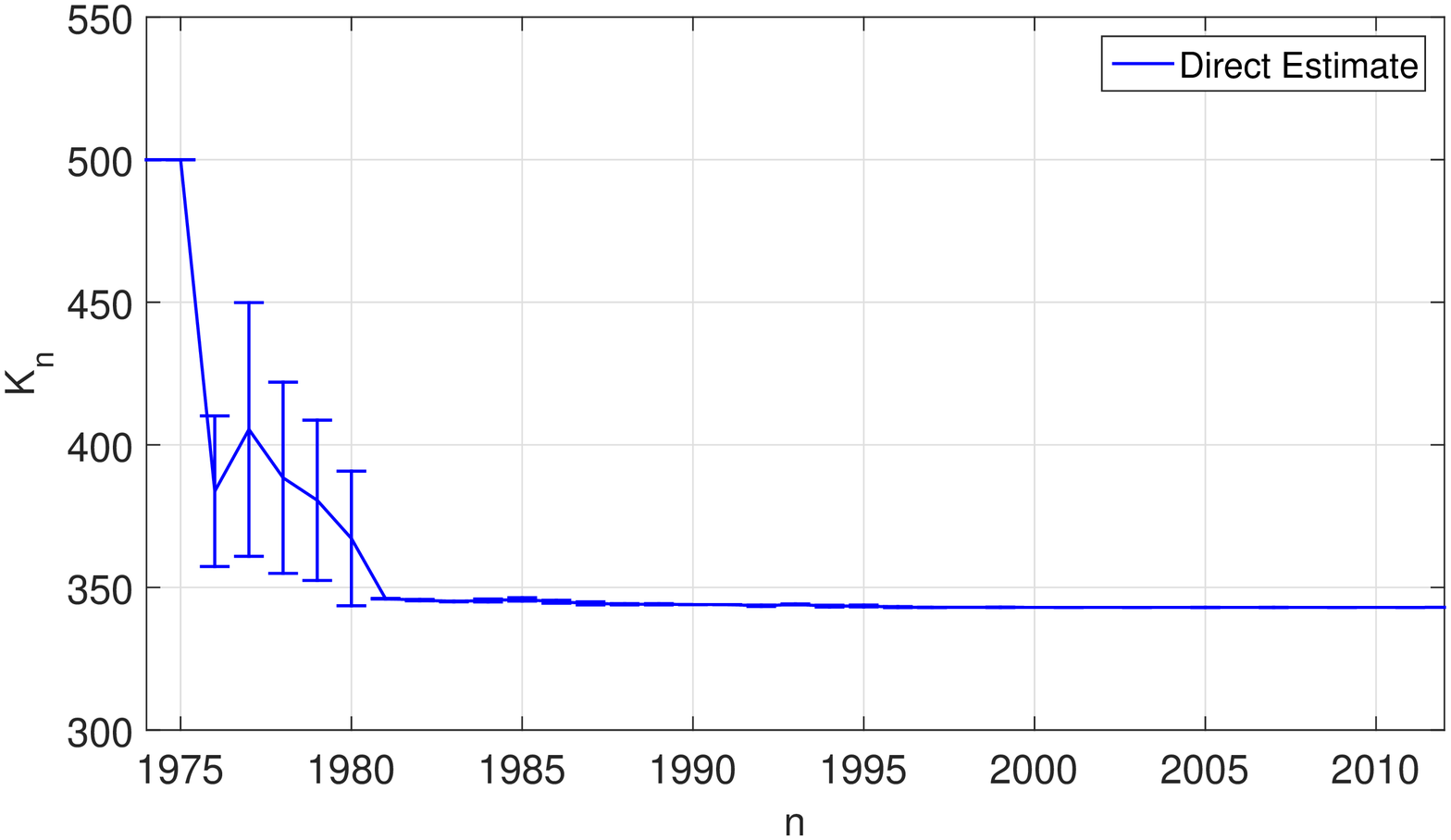}
		\caption{$K_{n}$}
		\label{exper:psid:Kn}
	\end{figure}
	\begin{figure}[!ht]
		\centering
		\includegraphics[width = \linewidth]{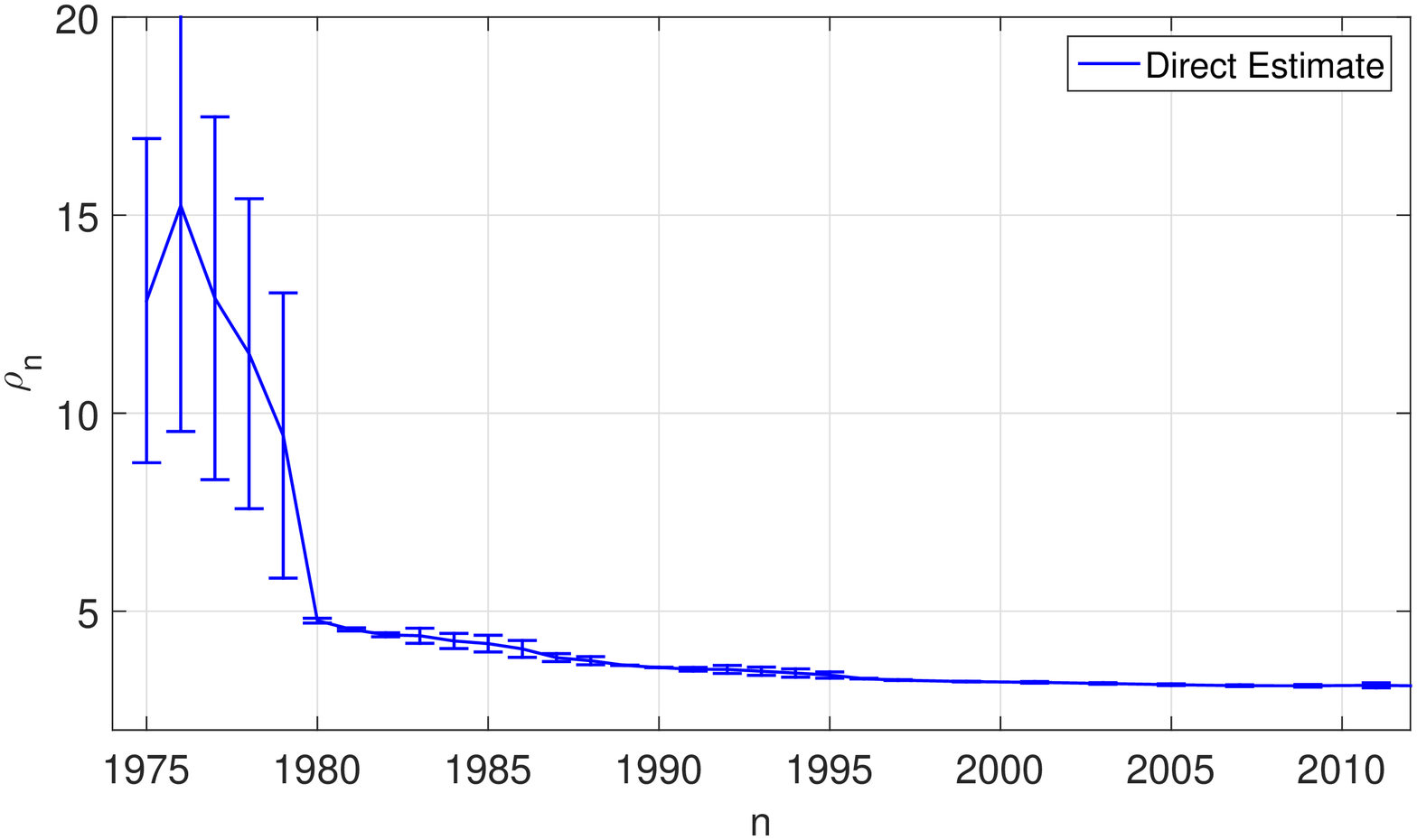}
		\caption{$\rho$ Estimate}
		\label{exper:psid:rhoEst}
	\end{figure}
	\begin{figure}[!ht]
		\centering
		\includegraphics[width = \linewidth]{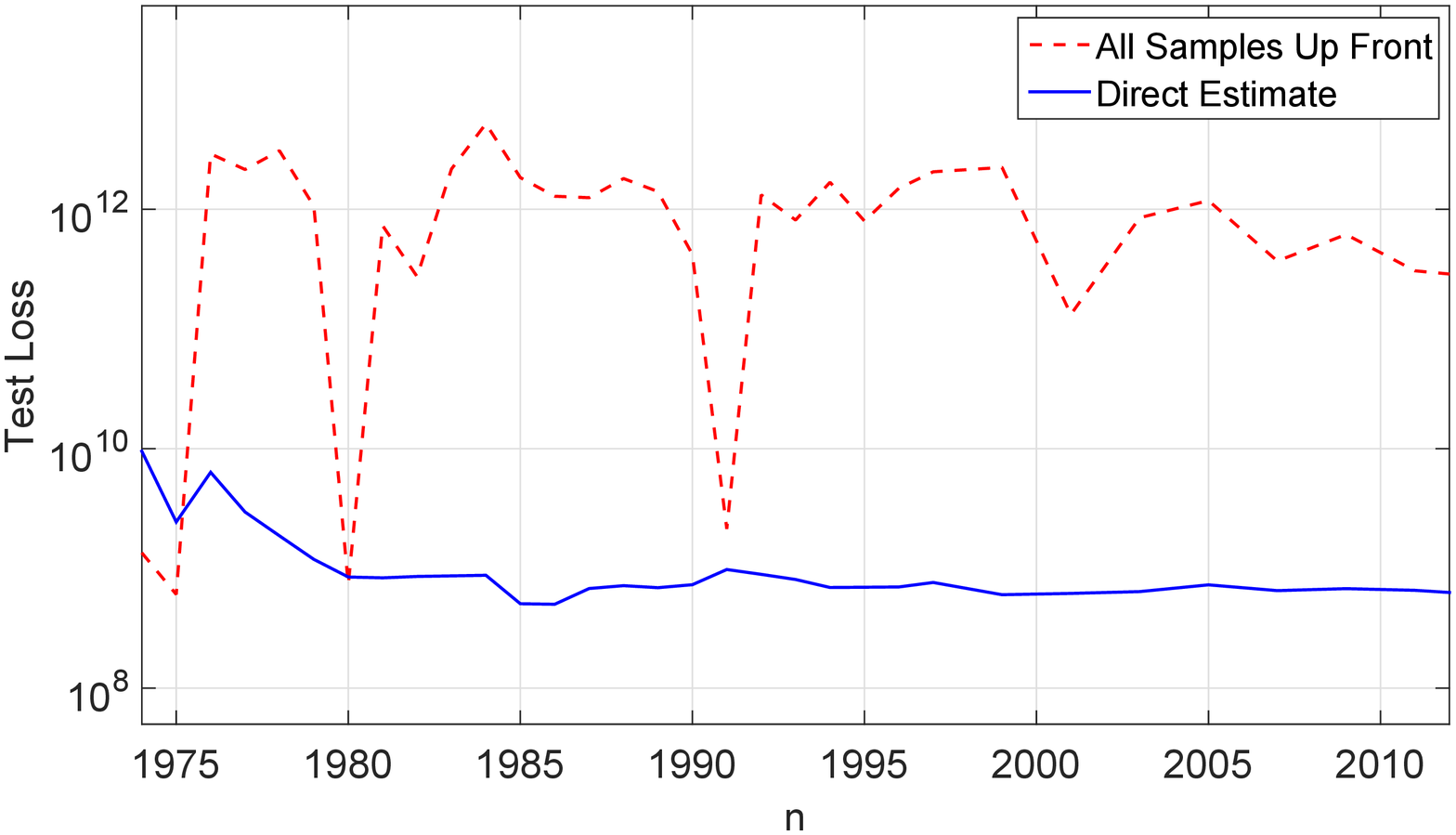}
		\caption{Test Loss}
		\label{exper:psid:testLosses}
	\end{figure}	
}{
\begin{figure}[!ht]
	\centering
	\begin{minipage}{0.5\linewidth}
		\centering
		\includegraphics[width = \linewidth]{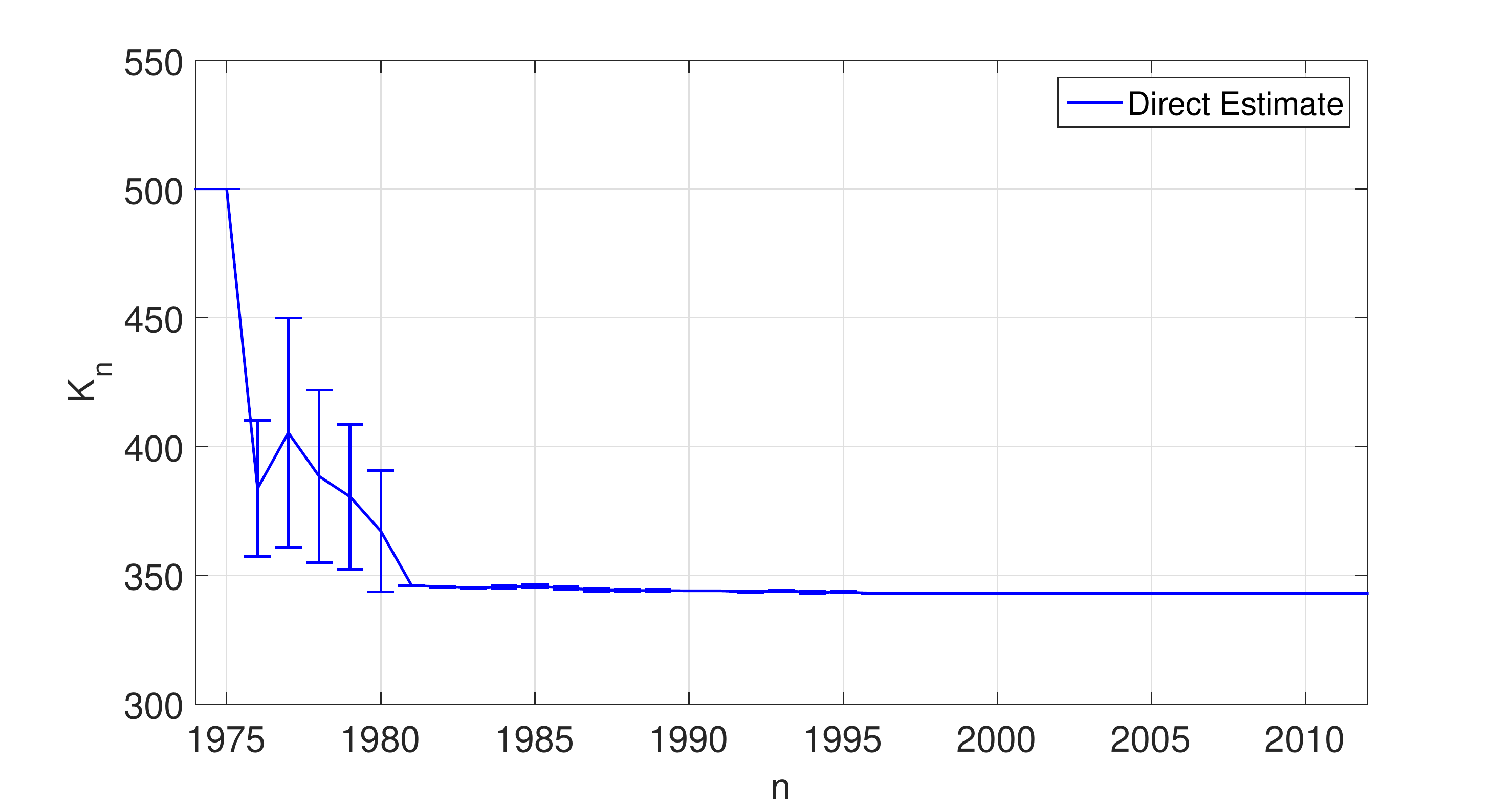}
		\caption{$K_{n}$}
		\label{exper:psid:Kn}
	\end{minipage}
\end{figure}
\begin{figure}[!ht]
	\centering
	\begin{minipage}{0.5\linewidth}
		\centering
		\includegraphics[width = \linewidth]{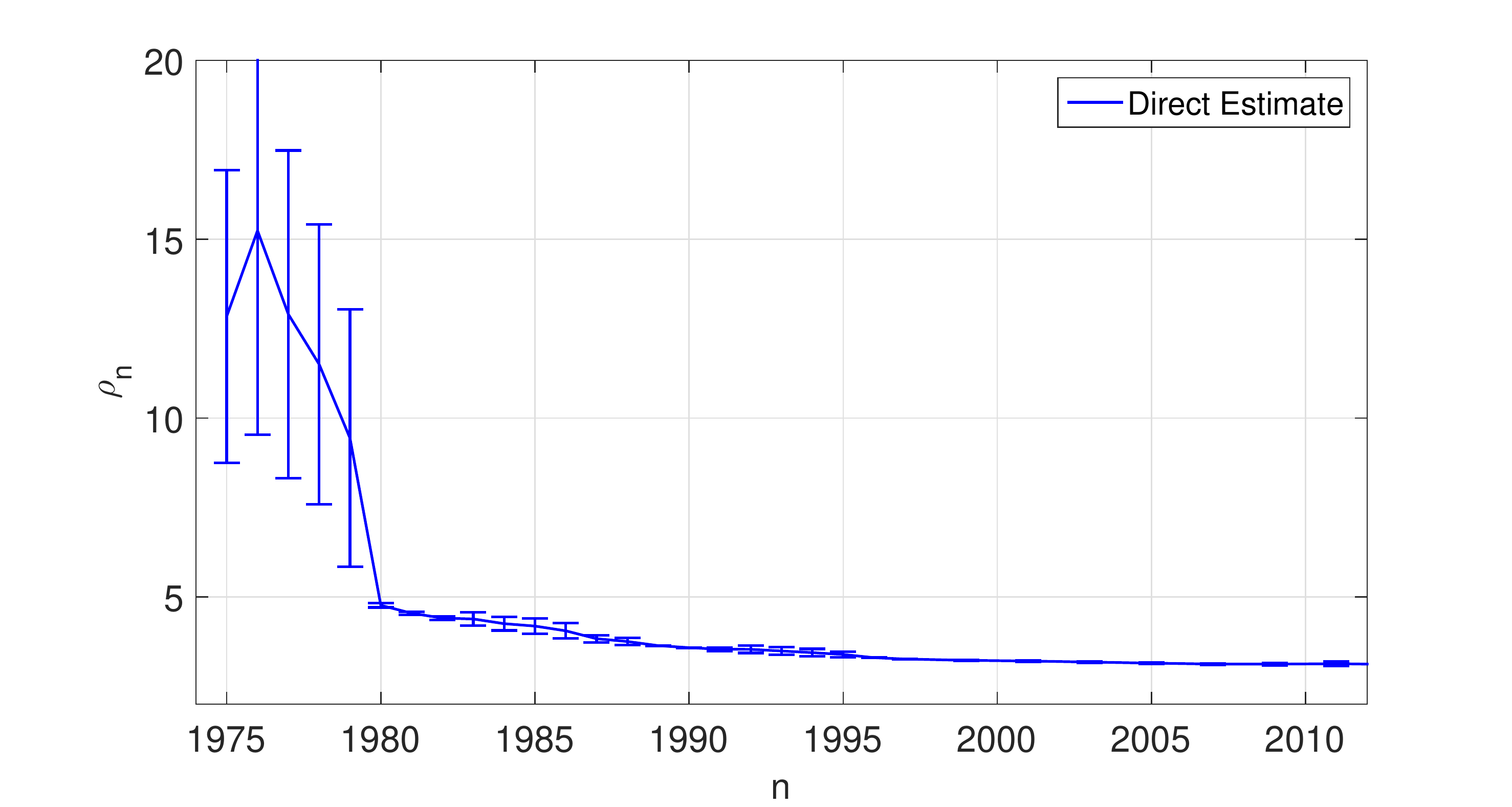}
		\caption{$\rho$ Estimate}
		\label{exper:psid:rhoEst}
	\end{minipage}%
	\begin{minipage}{0.5\linewidth}
		\centering
		\includegraphics[width = \linewidth]{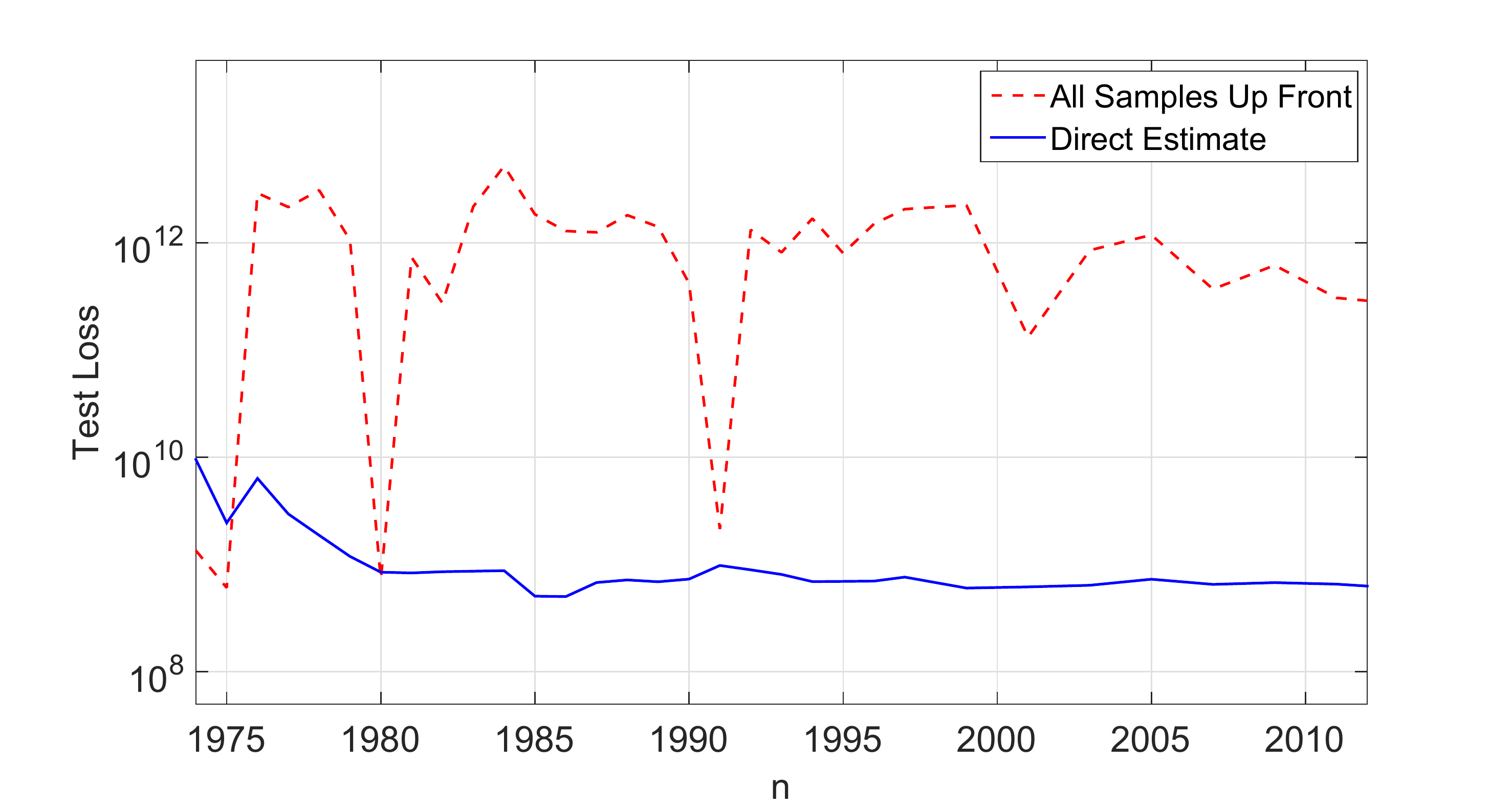}
		\caption{Test Loss}
		\label{exper:psid:testLosses}
	\end{minipage}%
\end{figure}	
}

\section{Conclusion}

We introduced a framework for adaptively solving a sequence of learning problems. We developed estimates of the change in the minimizers used to determine the number of training samples $K_{n}$ needed to achieve a target excess risk $\epsilon$. We introduced a cost based approach to select the number of samples and an approach to apply cross-validation. Experiments with synthetic and real data demonstrate that this approach is effective.

\bibliographystyle{IEEEbib}
\bibliography{ASO_ML_Journal}

\appendices

\end{document}